\documentclass[letterpaper,10pt]{article}
\usepackage{tabularx} 
\usepackage{graphicx} 
\usepackage[margin=1in]{geometry} 
\usepackage{cite} 
\usepackage[final]{hyperref} 
\hypersetup{
	colorlinks=true,       
	linkcolor=blue,        
	citecolor=blue,        
	filecolor=magenta,     
	urlcolor=blue         
}
\usepackage{indentfirst}
\usepackage{blindtext}
\usepackage{natbib}
\usepackage[utf8]{inputenc} 
\usepackage[T1]{fontenc}    
\usepackage{hyperref}       
\usepackage{url}            
\usepackage{booktabs}       
\usepackage{amsfonts}       
\usepackage{nicefrac}       
\usepackage{microtype}      
\usepackage{xcolor}         
\usepackage{dsfont}
\usepackage{subfig}
\usepackage{makecell}
\usepackage{amsmath}
\usepackage{amssymb}
\usepackage{mathtools}
\usepackage{amsthm}
\usepackage{multirow}
\usepackage{setspace}
\usepackage{booktabs}
\usepackage{algorithm}
\usepackage{algorithmic}
\theoremstyle{definition}

\usepackage{xspace}
\usepackage{longtable}
\usepackage{wrapfig}
\usepackage{enumitem}
\usepackage{comment}
\usepackage{etoc}
\setlist{leftmargin=5mm}
\usepackage[export]{adjustbox}

\theoremstyle{plain}
\newtheorem{theorem}{Theorem}[section]
\newtheorem{proposition}[theorem]{Proposition}

\newtheorem{corollary}[theorem]{Corollary}
\theoremstyle{definition}

\theoremstyle{remark}

\usepackage{todonotes}
\usepackage{tablefootnote}
\usepackage{xcolor}
\usepackage[capitalize,noabbrev]{cleveref}

\usepackage{amsmath,amsfonts,bm}

\def\eqref#1{Eqn.~(\ref{#1})}

\def\1{\bm{1}}

\def\vf{{\bm{f}}}

\def\vv{{\bm{v}}}
\def\vw{{\bm{w}}}
\def\vx{{\bm{x}}}

\def\vz{{\bm{z}}}

\DeclareMathAlphabet{\mathsfit}{\encodingdefault}{\sfdefault}{m}{sl}
\SetMathAlphabet{\mathsfit}{bold}{\encodingdefault}{\sfdefault}{bx}{n}

\def\gE{{\mathcal{E}}}

\def\gL{{\mathcal{L}}}

\def\gN{{\mathcal{N}}}

\def\sR{{\mathbb{R}}}

\newcommand{\E}{\mathbb{E}}

\usepackage{color}
\usepackage{bm}
\usepackage{hyperref}
\usepackage{amsmath,amsfonts}
\usepackage{blindtext}
\usepackage{listings}
\usepackage{mathtools}
\usepackage[normalem]{ulem}
\usepackage{multicol}
\usepackage{multirow}
\usepackage{graphicx}
\usepackage{extarrows}

\newcommand{\SB}{\text{SB}}

\newcommand{\data}{\text{data}}
\newcommand{\prior}{\text{prior}}

\newcommand{\bridge}{\text{bridge}}

\DeclarePairedDelimiterX{\infdivx}[2]{(}{)}{%
	#1\;\delimsize\|\;#2%
}

\newcommand{\kl}{D_{\mathrm{KL}}\infdivx}

\newcommand{\vect}[1]{\bm{#1}}

\newcommand{\x}{\xv}

\newcommand{\dm}{\mathrm{d}}

\newcommand{\epsilonv}{\vect\epsilon}

\newcommand{\muv}{\vect\mu}

\newcommand{\av}{\vect a}
\newcommand{\bv}{\vect b}

\newcommand{\fv}{\vect f}

\newcommand{\sv}{\vect s}

\newcommand{\wv}{\vect w}
\newcommand{\xv}{\vect x}

\newcommand{\zv}{\vect z}

\newcommand{\Fv}{\vect F}

\newcommand{\Iv}{\vect I}

\newcommand{\Hc}{\mathcal H}

\newcommand{\Lc}{\mathcal L}

\newcommand{\Nc}{\mathcal N}

\newcommand{\Uc}{\mathcal U}

\begin{document}

\title{\textbf{Schrodinger Bridges Beat Diffusion Models on Text-to-Speech Synthesis}}
\author{Zehua Chen\thanks{Equal contribution; \quad $^\dagger$Corresponding author: \texttt{dcszj@tsinghua.edu.cn}}$^{\ 1}$ \ Guande He$^{*1}$ \ Kaiwen Zheng$^{*1}$ \  
Xu Tan$^{2}$ \ Jun Zhu$^{\dagger1}$
}
\date{
\small $^1$Dept. of Comp. Sci. \& Tech., Institute for AI, THU-Bosch Joint Center for ML, Tsinghua University\\
$^2$Microsoft Research Asia
}
\maketitle

\vspace{-0.4cm}
\begin{abstract}
In text-to-speech (TTS) synthesis, diffusion models have achieved promising generation quality. However, because of the pre-defined data-to-noise diffusion process, their prior distribution is restricted to a noisy representation, which provides little information of the generation target.
In this work, we present a novel TTS system, Bridge-TTS, making the first attempt to substitute the noisy Gaussian prior in established diffusion-based TTS methods with a clean and deterministic one, which provides strong structural information of the target.
Specifically, we leverage the latent representation obtained from text input as our prior, and build a fully tractable Schrodinger bridge between it and the ground-truth mel-spectrogram, leading to a data-to-data process. 
Moreover, the tractability and flexibility of our formulation allow us to empirically study the design spaces such as noise schedules, as well as to develop stochastic and deterministic samplers.
Experimental results on the LJ-Speech dataset illustrate the effectiveness of our method in terms of both synthesis quality and sampling efficiency, significantly outperforming our diffusion counterpart Grad-TTS in 50-step/1000-step synthesis and strong fast TTS models in few-step scenarios. Project page: \url{https://bridge-tts.github.io/}.
\end{abstract}

\begin{figure}[ht]
    \centering
    \includegraphics[width=0.79\textwidth]{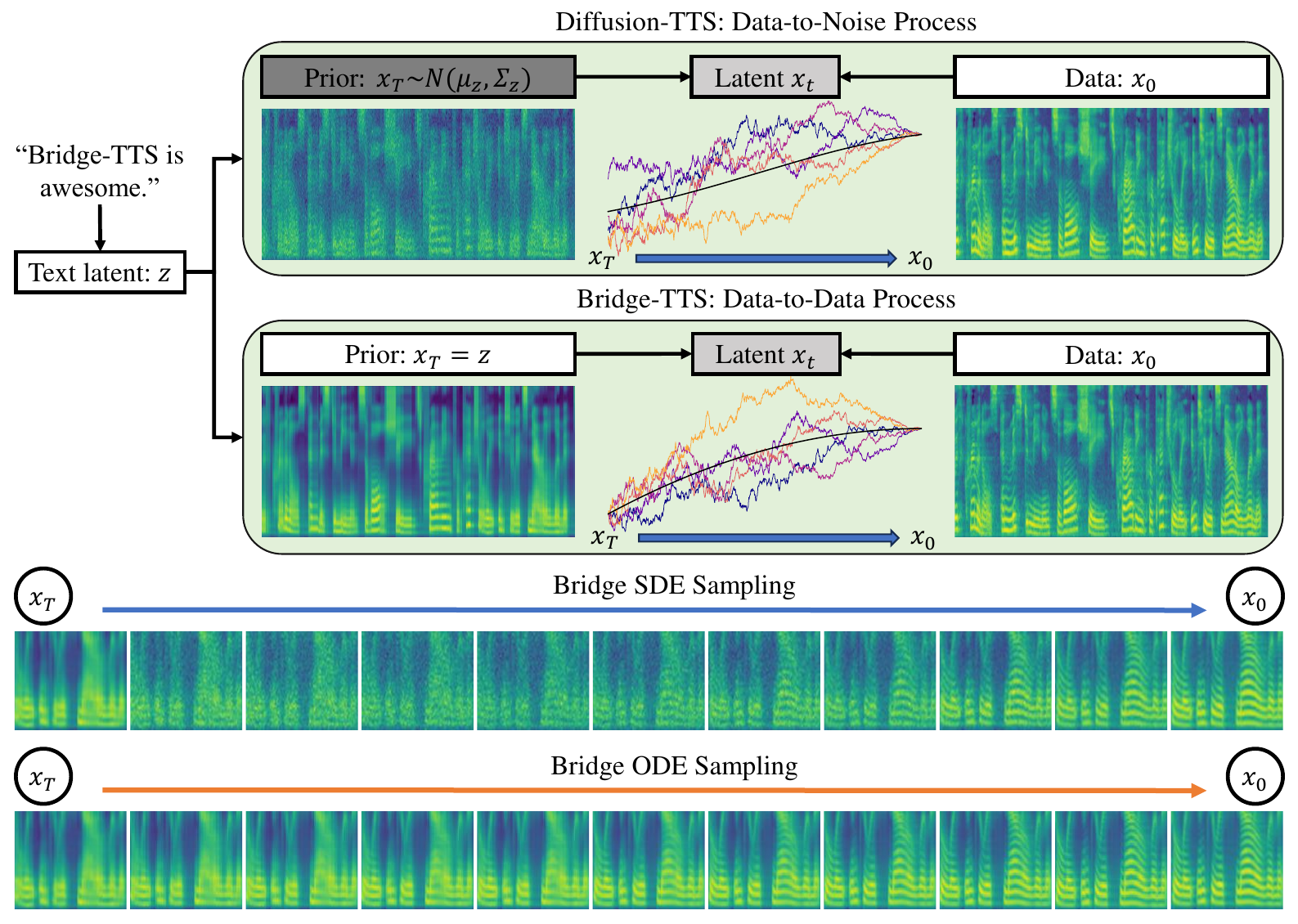}
    \caption{An overview of Bridge-TTS built on Schrodinger bridge.}
    \label{fig:flowchart}
\end{figure}

\newpage
\section{Introduction}
Diffusion models, including score-based generative models (SGMs) \citep{SGM} and denoising diffusion probabilistic models \citep{DDPM}, have been one of the most powerful generative models across different data generation tasks \citep{DALLE2,BinauralGrad,UniDiffuser,ProlificDreamer}. In speech community, they have been extensively studied in waveform synthesis \citep{DiffWave,WaveGrad,InferGrad}, text-to-audio generation \citep{AudioLDM,AudioLDM2,Make-An-Audio,Make-An-Audio2}, and text-to-speech (TTS) synthesis \citep{SurveyTan,Grad-TTS,NaturalSpeech2}. Generally, these models contain two processes between the data distribution and the prior distribution: 1) the forward diffusion process gradually transforms the data into a known prior distribution, e.g., Gaussian noise; 2) the reverse denoising process gradually generates data samples from the prior distribution.

In diffusion-based TTS systems \citep{Grad-TTS,LightGrad,CoMoSpeech}, the text input is usually first transformed into latent representation by a text encoder, which contains a phoneme encoder and a duration predictor, and then diffusion models are employed as a decoder to generate the mel-spectrogram conditioned on the latent. The prior distribution in these systems can be classified into two types: $1$) one is using the standard Gaussian noise to generate target \citep{ProDiff,DiffGAN-TTS,ResGrad}; $2$) the other improves the prior to be more informative of the target. For example, Grad-TTS \citep{Grad-TTS} learns the latent representation from the text encoder with the ground-truth target in training, and takes it as the mean of prior distribution to obtain a mean-shifted Gaussian. PriorGrad \citep{PriorGrad} utilizes the statistical values from training data, computing a Gaussian with covariance matrix. DiffSinger \citep{Diffsinger} employs an auxiliary model to acquire an intractable prior distribution, enabling a shallow reverse process. However, because diffusion models pre-specify the noise-additive diffusion process, the prior distribution of the above systems is confined to a noisy representation, which is not indicative of the mel-spectrogram.

In this work, as shown in Figure \ref{fig:flowchart}, we propose a new design to generate mel-spectrogram from a clean and deterministic prior, \textit{i.e.}, the text latent representation supervised by ground-truth target~\citep{Grad-TTS}. It has provided structural information of the target and is utilized as the condition information in both recent diffusion \citep{LightGrad,CoMoSpeech} and conditional flow matching \citep{VoiceFlow,MachaTTS} based TTS systems, while we argue that replacing the noisy prior in previous systems with this clean latent can further boost the TTS sample quality and inference speed. To enable this design, we leverage Schrodinger bridge~\citep{SB-1932,Diffusion-SB-Likelihood} instead of diffusion models, which seeks a data-to-data process rather than the data-to-noise process in diffusion models. As the original Schrodinger bridge is generally intractable that hinders the study of the design spaces in training and sampling, we propose a fully tractable Schrodinger bridge between paired data with a flexible form of reference SDE in alignment with diffusion models \citep{DDPM,SGM}.

With the tractability and flexibility of our proposed framework, aiming at TTS synthesis with superior generation quality and efficient sampling speed, we make an investigation of noise schedule, model parameterization, and training-free samplers, which diffusion models have greatly benefited from \citep{NoiseScheduleE2E,ProgressiveDis,DDIM}, while not been thoroughly studied in Schrodinger bridge related works. To summarize, we make the following key contributions in this work:

\begin{itemize}
    \item In TTS synthesis, we make the first attempt to generate the mel-spectrogram from clean text latent representation (\textit{i.e.}, the condition information in diffusion counterpart) by means of Schrodinger bridge, exploring data-to-data process rather than data-to-noise process.
    \item By proposing a fully tractable Schrodinger bridge between paired data with a flexible form of reference SDE, we theoretically elucidate and empirically explore the design spaces of noise schedule, model parameterization, and sampling process, further enhancing TTS quality with asymmetric noise schedule, data prediction, and first-order bridge samplers.
    \item Empirically, we attain both state-of-the-art generation quality and inference speed with a single training session. In both 1000-step and 50-step generation, we significantly outperform our diffusion counterpart Grad-TTS \citep{Grad-TTS};
    in 4-step generation, we accomplish higher quality than FastGrad-TTS \citep{FastGrad-TTS}; in 2-step generation, we surpass the state-of-the-art distillation method CoMoSpeech \citep{CoMoSpeech}, and the transformer-based model FastSpeech 2 \citep{Fastspeech2}.
\end{itemize}

\section{Background}

\subsection{Diffusion Models}
\label{sec:diffusionmodel}

Given a data distribution $p_{\data}(\boldsymbol{x})$, $\boldsymbol{x} \in \mathbb R^{d}$, SGMs \citep{SGM} are built on a continuous-time diffusion process defined by a forward stochastic differential equation (SDE):
\begin{align}
\label{eq:forwardSDE}
    \dm \x_t = \fv(\x_t, t) \dm t + g(t)\dm\wv_t,\quad \x_0 \sim p_0 = p_{\data}
\end{align}
where $t \in [0, T]$ for some finite horizon $T$, $\vf: \sR^d \times [0,T] \to \sR^d$ is a vector-valued drift term, $g:[0,T] \to \sR$ is a scalar-valued diffusion term, and $\wv_t \in \sR^d$ is a standard Wiener process. Under proper construction of $\fv,g$, the boundary distribution $p_T(\boldsymbol{x}_T)$ is approximately a Gaussian prior distribution $p_{\prior}=\Nc(\vect0,\sigma_T^2\Iv)$. The forward SDE has a corresponding reverse SDE~\citep{SGM} which shares the same marginal distributions $\{p_t\}_{t=0}^T$ with the forward SDE: 
\begin{align}
    \label{eq:reverseSDE}
    \dm\x_t = [\fv(\x_t, t) - g^2(t)\nabla \log p_t(\x_t)] \dm t + g(t)\dm \bar{\wv}_t,\quad \x_T\sim p_T\approx p_{\prior}
\end{align}
where $\bar{\boldsymbol{w}}_t$ is the reverse-time Wiener process, and the only unknown term $\nabla\log p_t(\boldsymbol x_t)$ is the \textit{score function} of the marginal density $p_t$. By parameterizing a score network $\sv_\theta(\x_t,t)$ to predict $\nabla\log p_t(\boldsymbol x_t)$, we can replace the true score in \eqref{eq:reverseSDE} and solve it reversely from $p_{\prior}$ at $t=T$, yielding generated data samples at $t=0$. $\sv_\theta(\x_t,t)$ is usually learned by the denoising score matching (DSM) objective~\citep{VincentScoreMatching,SGM} with a weighting function $\lambda(t)>0$:
\begin{align}
    \label{eq:DSM}
     \mathbb E_{p_0(\boldsymbol{x}_0)p_{t|0}(\boldsymbol{x}_t|\boldsymbol{x}_0)}
    \E_t\left[\lambda(t)\|s_{\theta}(\boldsymbol{x}_t, t)- \nabla \log p_{t|0}(\boldsymbol{x}_t|\boldsymbol{x}_0) \|_2^2\right],
\end{align}
where $t \sim \Uc(0, T)$ and $p_{t|0}$ is the conditional transition distribution from $\boldsymbol{x}_0$ to $\boldsymbol{x}_t$, which is determined by the pre-defined forward SDE and is analytical for a linear drift $\fv(\x_t,t)=f(t)\x_t$.

\subsection{Diffusion-Based TTS Systems}
The goal of TTS systems is to learn a generative model $p_{\theta}(\vx|y)$ over mel-spectrograms (Mel) $\vx \in \sR^d$ given conditioning text $y_{1:L}$ with length $L$. 
Grad-TTS \citep{Grad-TTS} provides a strong baseline for TTS with SGM, which consists of a text encoder and a diffusion-based decoder. 
Specifically, they alter the Gaussian prior in SGMs to another one $\tilde {p}_{\text{enc}}(\vz | y) = \gN (\vz, \Iv)$ with informative mean $\zv$, where $\vz \in \sR^d$ is a latent acoustic feature transformed from a text string $y$ through the text encoder network $\gE$, i.e., $\vz = \gE(y)$. 
The diffusion-based decoder utilizes $\tilde{p}_{\text{enc}}$ as prior for SGM and builds a diffusion process via the following modified forward SDE:
\begin{align}
    \label{eq:forwardSDE-GradTTS}
    \dm\boldsymbol{x}_t &= \tfrac{1}{2}(\boldsymbol{z} - \boldsymbol{x}_t)\beta_t \dm t + \sqrt{\beta_t} \dm\boldsymbol{w}_t,\quad \vx_0 \sim p_0 = p_\data(\vx|y)
\end{align}
where $p_0=p_\data(\vx|y)$ is the true conditional data distribution and $\beta_t$ is a non-negative noise schedule. The forward SDE in \eqref{eq:forwardSDE-GradTTS} will yield $\vx_T \sim p_T \approx \tilde p_{\text{enc}}$ with sufficient large $T$~\citep{Grad-TTS}. 
During training, the text encoder and the diffusion-based decoder are jointly optimized, where the encoder is optimized with a negative log-likelihood loss $\mathcal{L}_{\text{enc}} = -\E_{p_\data(\vx|y)}[\log \tilde{p}_{\text{enc}}(\vx|y)]$ and the decoder is trained with the DSM objective in \eqref{eq:DSM}, denoted as $\mathcal L_{\text{diff}}$. Apart from $\gL_{\text{enc}}$ and $\gL_{\text{diff}}$, the TTS system also optimizes a duration predictor $\hat A$ as a part of the encoder that predicts the alignment map $A^*$ between encoded text sequence $\tilde{\boldsymbol{z}}_{1:L}$ and the latent feature $\boldsymbol{z}_{1:F}$ with $F$ frames given by Monotonic Alignment Search \citep{kim2020glow}, where $\boldsymbol{z}_j = \tilde{\boldsymbol{z}}_{A^*(j)}$. Denote the duration prediction loss as $\gL_{\text{dp}}$, the overall training objective of Grad-TTS is $\mathcal {L}_{\text{grad-tts}} = \mathcal{L}_{\text{enc}} + \gL_{\text{dp}} + \mathcal{L}_{\text{diff}}$.

\subsection{Schrodinger Bridge}
\label{sec:Sbp}

The Schrodinger Bridge (SB) problem~\citep{SB-1932,Diffusion-SB-IPF, Diffusion-SB-Likelihood} originates from the optimization of path measures with constrained boundaries:
\begin{equation}
\label{eq:SB-DEF}
    \min_{p\in\mathcal P_{[0,T]}}\kl{p}{p^{\text{ref}}},\quad \textit{s.t.}\,\, p_0=p_{\data},p_T=p_{\prior}
\end{equation}
where $\mathcal P_{[0,T]}$ is the space of path measures on a finite time horizon $[0,T]$, $p^{\text{ref}}$ is the \textit{reference path measure}, and $p_0,p_T$ are the marginal distributions of $p$ at boundaries. Generally, $p^{\text{ref}}$ is defined by the same form of forward SDE as SGMs in~\eqref{eq:forwardSDE} (i.e., the \textit{reference SDE}). In such a case, the SB problem is equivalent to a couple of forward-backward SDEs~\citep{SB-DGM,Diffusion-SB-Likelihood}:
\begin{subequations}\label{eq:SB-SDE}
\begin{align}
    &\dm\x_t = [\fv(\x_t, t) + g^2(t) \nabla \log \Psi_t(\x_t) ] \dm t + g(t)\dm\wv_t, \quad \x_0\sim p_{\data} \label{eq:SB-forwardSDE} \\ 
    &\dm\x_t = [\fv(\x_t, t) - g^2(t) \nabla \log \widehat{\Psi}_t(\x_t) ] \dm t + g(t)\dm\bar{\wv}_t,\quad \x_T\sim p_{\prior} \label{eq:SB-reverseSDE}
\end{align}
\end{subequations}
where $\fv$ and $g$ are the same as in the reference SDE. The extra non-linear drift terms $\nabla \log \Psi_t(\boldsymbol x_t)$ and $\nabla \log \widehat{\Psi}_t(\boldsymbol x_t)$ are also described by the following coupled partial differential equations (PDEs):
\begin{equation}
\label{eq:SB-PDE}
\begin{cases}
    \frac{\partial \Psi}{\partial t}=-\nabla_{\boldsymbol{x}} \Psi^{\top} \boldsymbol{f}-\frac{1}{2} \operatorname{Tr}\left(g^2 \nabla_{\boldsymbol{x}}^2 \Psi\right) \\
    \frac{\partial \widehat{\Psi}}{\partial t}=-\nabla_{\boldsymbol{x}} \cdot(\widehat{\Psi} \boldsymbol{f})+\frac{1}{2} \operatorname{Tr}\left(g^2 \nabla_{\boldsymbol{x}}^2 \widehat{\Psi}\right)
\end{cases}
\quad\textit{s.t.}\,\, \Psi_0\widehat\Psi_0 = p_{\data}, \Psi_T\widehat\Psi_T = p_{\prior}.
\end{equation}
The marginal distribution $p_t$ of the SB at any time $t\in [0,T]$ satisfies $p_t=\Psi_t\widehat\Psi_t$. Compared to SGMs where $p_T\approx p_{\prior}=\Nc(\muv,\sigma_T^2\Iv)$, SB allows for a flexible form of $p_{\prior}$ and ensures the boundary condition $p_T=p_{\prior}$.
However, solving the SB requires simulating stochastic processes and performing costly iterative procedures~\citep{Diffusion-SB-IPF, Diffusion-SB-Likelihood,DSBM}. Therefore, it suffers from scalability and applicability issues. 
In certain scenarios, such as using paired data as boundaries, the SB problem can be solved in a simulation-free approach~\citep{AlignedDSB,I2SB}. 
Nevertheless, SBs in these works are either not fully tractable or limited to restricted families of $p^{\text{ref}}$, thus lacking a comprehensive and theoretical analysis of the design spaces.

\section{Bridge-TTS}
We extend SB techniques to the TTS task and elucidate the design spaces with theory-grounded analyses. We start with a fully tractable SB between paired data in TTS modeling. Based on such formulation, we derive different training objectives and theoretically study SB sampling in the form of SDE and ODE, which lead to novel first-order sampling schemes when combined with exponential integrators. In the following discussions, we say two probability density functions are the same when they are up to a normalizing factor. Besides, we assume the maximum time $T=1$ for convenience.
\subsection{Schrodinger Bridge between Paired Data}

As we have discussed, with the properties of unrestricted prior form and strict boundary condition, SB is a natural substitution for diffusion models when we have a strong informative prior. 
In the TTS task, the pairs of the ground-truth data $(\x,y)$ and the deterministic prior $\zv=\gE(y)$ given by the text encoder can be seen as mixtures of dual Dirac distribution boundaries $(\delta_{\x},\delta_{\zv})$, which simplifies the solving of SB problem. However, in such a case, the SB problem in~\eqref{eq:SB-DEF} will inevitably collapse given a stochastic reference process that admits a continuous density 
$p^{\text{ref}}_1$ at $t=1$, since the KL divergence between a Dirac distribution and a continuous probability measure is infinity.

To tackle this problem, we consider a noisy observation of boundary data points $\x_0,\x_1$ polluted by a small amount of Gaussian noise $\Nc(\vect 0,\epsilon_1^2\Iv)$ and $\Nc(\vect 0,\epsilon_2^2\Iv)$ respectively, which helps us to identify the SB formulation between clean data when $\epsilon_1,\epsilon_2\rightarrow 0$. Actually, we show that in general cases where the reference SDE has a linear drift $\fv(\x_t,t)=f(t)\x_t$ (which is aligned with SGMs), SB has a fully tractable and neat solution when $\epsilon_2=e^{\int_0^1f(\tau)\dm\tau}\epsilon_1$. We formulate the result in the following theorem.

\begin{proposition}[Tractable Schrodinger Bridge between Gaussian-Smoothed Paired Data with Reference SDE of Linear Drift, proof in Appendix~\ref{appendix:proof-tractable-sb}]
\label{prop:tractable-sb}
    Assume $\fv=f(t)\x_t$, the analytical solution to~\eqref{eq:SB-PDE} when $p_{\data}=\Nc(\x_0,\epsilon^2\Iv)$ and $p_{\prior}=\Nc(\x_1,e^{2\int_0^1f(\tau)\dm\tau}\epsilon^2\Iv)$ is
    \begin{equation}
\widehat{\Psi}_t^{\epsilon}=\Nc(\alpha_t\av,(\alpha_t^2\sigma^2+\alpha_t^2\sigma_t^2)\Iv),\quad \Psi_t^{\epsilon}=\Nc(\bar\alpha_t\bv,(\alpha_t^2\sigma^2+\alpha_t^2\bar\sigma_t^2)\Iv)
    \end{equation}
    where $t\in [0,1]$,
    \begin{equation}
        \av=\x_0+\frac{\sigma^2}{\sigma_1^2}(\x_0-\frac{\x_1}{\alpha_1}),\quad \bv=\x_1+\frac{\sigma^2}{\sigma_1^2}(\x_1-\alpha_1\x_0),\quad \sigma^2=\epsilon^2+\frac{\sqrt{\sigma_1^4+4\epsilon^4}-\sigma_1^2}{2},
    \end{equation}
    and
    \begin{equation}
\alpha_t=e^{\int_0^tf(\tau)\dm\tau},\quad\bar\alpha_t=e^{-\int_t^1f(\tau)\dm\tau},\quad\sigma_t^2 = \int_0^t \frac{g^2(\tau)}{\alpha_\tau^2} \dm\tau,\quad \bar\sigma_t^2 = \int_t^1 \frac{g^2(\tau)}{\alpha_\tau^2} \dm\tau.
    \end{equation}
\end{proposition}
In the above theorem, $\alpha_t,\bar\alpha_t,\sigma_t,\bar\sigma_t$ are determined by $f,g$ in the reference SDE (\eqref{eq:forwardSDE}) and are analogous to the \textit{noise schedule} in SGMs~\citep{VariationalDiffusion}. 
When $\epsilon\rightarrow 0$, $\widehat{\Psi}_t^{\epsilon},\Psi_t^{\epsilon}$ converge to the tractable solution between clean paired data $(\x_0,\x_1)$:
\begin{equation}
\label{eq:tractable_psi_limit}\widehat{\Psi}_t=\Nc(\alpha_t\x_0,\alpha_t^2\sigma_t^2\Iv),\quad \Psi_t=\Nc(\bar\alpha_t\x_1,\alpha_t^2\bar\sigma_t^2\Iv)
\end{equation}
The advantage of such tractability lies in its ability to facilitate the study of training and sampling under the forward-backward SDEs (\eqref{eq:SB-SDE}), which we will discuss in the following sections. Besides, the marginal distribution $p_t=\widehat{\Psi}_t\Psi_t$ of the SB  also has a tractable form: 
\begin{equation}
\label{eq:tractable_marginal}
p_t=\Psi_t\widehat\Psi_t=\Nc\left(\frac{\alpha_t\bar\sigma_t^2\x_0+\bar\alpha_t\sigma_t^2\x_1}{\sigma_1^2},\frac{\alpha_t^2\bar\sigma_t^2\sigma_t^2}{\sigma_1^2}\Iv\right), 
\end{equation}
which is a Gaussian distribution whose mean is an interpolation between $\x_0,\x_1$, and variance is zero at boundaries and positive at the middle. A special case is that, when the noise schedule $f(t)=0$ and $g(t)=\sigma>0$, we have $p_t=\Nc((1-t)\x_0+t\x_1,\sigma^2t(1-t)\Iv)$, which recovers the Brownian bridge used in previous works~\citep{SE-Bridge,tong2023simulation,tong2023improving}. Actually, \eqref{eq:tractable_marginal} reveals the form of \textit{generalized Brownian bridge with linear drift and time-varying volatility} between $\x_0$ and $\x_1$. We put the detailed analysis in Appendix~\ref{appendix:brownian-bridge}.

\subsection{Model Training}
\label{sec:model_training}

The TTS task aims to learn a model to generate the Mel $\x_0$ given text $y$. Denote $\x_1=\mathcal{E}(y)$ as the latent acoustic feature produced by text encoder $\mathcal{E}$, since the SB is tractable given $\x_0,\x_1$ ($\nabla\log\Psi,\nabla\log\widehat\Psi$ in~\eqref{eq:SB-SDE} are determined by~\eqref{eq:tractable_psi_limit}), a direct training approach is to parameterize a network $\x_\theta$ to predict $\x_0$ given $\x_t$ at different timesteps, which allows us to simulate the process of SB from $t=1$ to $t=0$. This is in alignment with the \textit{data prediction} in diffusion models, and we have the bridge loss:
\begin{equation}
    \label{eq:data-predictor}
    \Lc_{\bridge}=\E_{(\x_0,y)\sim p_{\data},\x_1=\mathcal{E}(y)}\E_t[\|\x_\theta(\x_t,t,\x_1)-\x_0\|_2^2]
\end{equation}
where $\x_t=\frac{\alpha_t\bar\sigma_t^2}{\sigma_1^2}\x_0+\frac{\bar\alpha_t\sigma_t^2}{\sigma_1^2}\x_1+\frac{\alpha_t\bar\sigma_t\sigma_t}{\sigma_1}\epsilonv,\epsilonv\sim\Nc(\vect0,\Iv)$ by the SB (\eqref{eq:tractable_marginal}). $\x_1$ is also fed into the network as a condition, following Grad-TTS~\citep{Grad-TTS}.

Analogous to the different parameterizations in diffusion models, there are alternative choices of training objectives that are equivalent in bridge training, such as the noise prediction corresponding to $\nabla\log \widehat\Psi_t$~\citep{I2SB} or the SB score $\nabla\log p_t$, and the velocity prediction related to flow matching techniques~\citep{FlowMatching}. However, we find they perform worse or poorly in practice, which we will discuss in detail in Appendix~\ref{appendix:training-objective}. 
Except for the bridge loss, we jointly train the text encoder $\mathcal{E}$ (including the duration predictor $\hat A$) following Grad-TTS. Since the encoder no longer parameterizes a Gaussian distribution, we simply adopt an MSE encoder loss $\mathcal L^\prime_{\text{enc}} = \E_{(\x_0,y)\sim p_{\data}}\| \mathcal{E}(y) - \x_0 \|^2$. And we use the same duration prediction loss $\gL_{\text{dp}}$ as Grad-TTS.
The overall training objective of Bridge-TTS is $\mathcal {L}_{\text{bridge-tts}} = \mathcal{L}_{\text{enc}}^\prime + \gL_{\text{dp}} +  \mathcal{L}_{\text{bridge}}$.

In our framework, the flexible form of reference SDE facilitates the design of noise schedules $f,g$, which constitutes an important factor of performance as in SGMs. In this work, we directly transfer the well-behaved noise schedules from SGMs, such as variance preserving (VP). As shown in Table~\ref{tab:noise_schedule}, we set $f,g^2$ linear to $t$, and the corresponding $\alpha_t,\sigma_t^2$ have closed-form expressions. Such designs are new in both SB and TTS-related contexts and distinguish our work from previous ones with Brownian bridges~\citep{SE-Bridge,tong2023simulation,tong2023improving}.

\begin{table}[H]
\caption{\label{tab:noise_schedule}Demonstration of the noise schedules in Bridge-TTS.}
\label{mspeed}
\small
\begin{center}
\resizebox{0.98\linewidth}{!}{
\begin{tabular}{c|cccc}
\toprule 
Schedule  & $f(t)$ & $g^2(t)$ & $\alpha_t$ & $\sigma^2_t$   \\
\midrule
Bridge-gmax\tablefootnote{The main hyperparameter for the Bridge-gmax schedule is $\beta_1$, which is exactly the maximum of $g^2(t)$.} & 0 & $\beta_0 + t(\beta_1 - \beta_0)$ & 1 & $\frac{1}{2}(\beta_1 - \beta_{0})t^2 + \beta_{0}t$  \\ \addlinespace[4pt]
Bridge-VP & $-\frac{1}{2}(\beta_0 + t(\beta_1 - \beta_0))$ & $\beta_0 + t(\beta_1 - \beta_0)$ & $e^{-\frac{1}{2} \int_0^t (\beta_{0} + \tau(\beta_{1} - \beta_{0}))\dm\tau}$ & $e^{\int_0^t (\beta_0 + \tau(\beta_1 - \beta_0))\dm\tau} - 1$ \\
\bottomrule
\end{tabular}
}
\end{center}
\end{table}

\subsection{Sampling Scheme}
Assume we have a trained data prediction network $\x_\theta(\x_t,t)$ \footnote{We omit the condition $\x_1$ for simplicity and other parameterizations such as noise prediction can be first transformed to $\x_\theta$.}. If we replace $\x_0$ with $\x_\theta$ in the tractable solution of $\widehat\Psi,\Psi$ (\eqref{eq:tractable_psi_limit}) and substitute them into~\eqref{eq:SB-SDE}, which describes the SB with SDEs, we can obtain the parameterized SB process. 
Analogous to the sampling in diffusion models, the parameterized SB can be described by both stochastic and deterministic processes, which we call bridge SDE/ODE, respectively. 
\paragraph{Bridge SDE}
We can follow the reverse SDE in~\eqref{eq:SB-reverseSDE}. By substituting~\eqref{eq:tractable_psi_limit} into it and replace $\x_0$ with $\x_\theta$, we have the \textit{bridge SDE}:
 \begin{equation}
 \label{eq:bridge-SDE}
\dm\x_t=\left[f(t)\x_t+g^2(t)\frac{\x_t-\alpha_t\x_\theta(\x_t,t)}{\alpha_t^2\sigma_t^2}\right]\dm t+g(t)\dm\bar\wv_t
 \end{equation}
\paragraph{Bridge ODE}
The \textit{probability flow ODE}~\citep{SGM} of the forward SDE in~\eqref{eq:SB-forwardSDE} is~\citep{Diffusion-SB-Likelihood}:
    \begin{equation}
    \begin{aligned}
       \dm\boldsymbol{x}_t &= \left[\fv(t)\x_t + g^2(t) \nabla \log \Psi_t(\x_t)-\frac{1}{2}g^2(t) \nabla \log p_t(\x_t)\right] \dm t\\
        &=\left[\fv(t)\x_t + \frac{1}{2}g^2(t) \nabla \log \Psi_t(\boldsymbol x_t)-\frac{1}{2}g^2(t) \nabla \log \widehat\Psi_t(\boldsymbol x_t)\right] \dm t
    \end{aligned}
    \end{equation}

where we have used $\nabla \log p_t(\x_t)=\nabla\log \Psi_t(\x_t)+\nabla\log \widehat\Psi_t(\x_t)$ since $p_t=\Psi_t\widehat\Psi_t$. By substituting~\eqref{eq:tractable_psi_limit} into it and replace $\x_0$ with $\x_\theta$, we have the \textit{bridge ODE}:
 \begin{equation}
  \label{eq:bridge-ODE}
     \dm\boldsymbol{x}_t = \left[\fv(t)\x_t - \frac{1}{2}g^2(t) \frac{\x_t-\bar\alpha_t\x_1}{\alpha_t^2\bar\sigma_t^2}+\frac{1}{2}g^2(t)\frac{\x_t-\alpha_t\x_\theta(\x_t,t)}{\alpha_t^2\sigma_t^2}\right] \dm t
 \end{equation}
To obtain data sample $\xv_0$, we can solve the bridge SDE/ODE from the latent $\x_1$ at $t=1$ to $t=0$. However, directly solving the bridge SDE/ODE may cause large errors when the number of steps is small. A prevalent technique in diffusion models is to handle them with exponential integrators~\citep{DPM-Solver,DPM-Solver++,zheng2023dpm,SEEDS}, which aims to ``cancel'' the linear terms involving $\x_t$ and obtain solutions with lower discretization error. We conduct similar derivations for bridge sampling, and present the results in the following theorem.
\begin{proposition}[Exact Solution and First-Order Discretization of Bridge SDE/ODE, proof in Appendix~\ref{appendix:proof-sb-sampling}]
\label{prop:sampler}
Given an initial value $\x_s$ at time $s>0$, the solution at time $t\in [0,s]$ of bridge SDE/ODE is
\begin{align}
\label{eq:exact-solution-SDE}
    \x_t&=\frac{\alpha_t\sigma_t^2}{\alpha_s\sigma_s^2}\x_s-\alpha_t\sigma_t^2\int_s^t\frac{g^2(\tau)}{\alpha_\tau^2\sigma_\tau^4}\x_\theta(\x_\tau,\tau)\dm\tau+\alpha_t\sigma_t\sqrt{1-\frac{\sigma_t^2}{\sigma_s^2}}\epsilonv,\quad \epsilonv\sim\Nc(\vect 0,\Iv)\\
    \label{eq:exact-solution-ODE}
    \x_t&=\frac{\alpha_t\sigma_t\bar\sigma_t}{\alpha_s\sigma_s\bar\sigma_s}\x_s+\frac{\bar\alpha_t\sigma_t^2}{\sigma_1^2}\left(1-\frac{\sigma_s\bar\sigma_t}{\bar\sigma_s\sigma_t}\right)\x_1-\frac{\alpha_t\sigma_t\bar\sigma_t}{2}\int_s^t\frac{g^2(\tau)}{\alpha_\tau^2\sigma_\tau^3\bar\sigma_\tau}\x_\theta(\x_\tau,\tau)\dm\tau
\end{align}
  The first-order discretization (with the approximation $\x_\theta(\x_\tau,\tau)\approx \x_\theta(\x_s,s)$ for $\tau\in [t,s]$) gives
  \begin{align}
      \label{eq:first-order-SDE}
    \x_t&=\frac{\alpha_t\sigma_t^2}{\alpha_s\sigma_s^2}\x_s+\alpha_t\left(1-\frac{\sigma_t^2}{\sigma_s^2}\right)\x_\theta(\x_s,s)+\alpha_t\sigma_t\sqrt{1-\frac{\sigma_t^2}{\sigma_s^2}}\epsilonv,\quad \epsilonv\sim\Nc(\vect 0,\Iv)\\
     \label{eq:first-order-ODE}
\x_t&=\frac{\alpha_t\sigma_t\bar\sigma_t}{\alpha_s\sigma_s\bar\sigma_s}\x_s+\frac{\alpha_t}{\sigma_1^2}\left[\left(\bar\sigma_t^2-\frac{\bar\sigma_s\sigma_t\bar\sigma_t}{\sigma_s}\right)\x_\theta(\x_s,s)+\left(\sigma_t^2-\frac{\sigma_s\sigma_t\bar\sigma_t}{\bar\sigma_s}\right)\frac{\x_1}{\alpha_1}\right]
  \end{align}
\end{proposition}
To the best of our knowledge, such derivations are revealed for the first time in the context 
of SB. We find that the first-order discretization of bridge SDE (\eqref{eq:first-order-SDE}) recovers posterior sampling~\citep{I2SB} on a Brownian bridge, and the first-order discretization of bridge ODE (\eqref{eq:first-order-ODE}) in the limit of $\frac{\sigma_s}{\sigma_1},\frac{\sigma_t}{\sigma_1}\rightarrow 0$ recovers deterministic DDIM sampler~\citep{DDIM} in diffusion models. Besides, we can easily discover that the 1-step case of~\eqref{eq:first-order-SDE} and~\eqref{eq:first-order-ODE} are both 1-step deterministic prediction by $\x_\theta$. We put more detailed analyses in Appendix~\ref{appendix:posterior-sampling}.

We can also develop higher-order samplers by taking higher-order Taylor expansions for $\x_\theta$ in the exact solutions. We further discuss and take the predictor-corrector method as the second-order case in Appendix~\ref{appendix:high-order-sampler}. In practice, we find first-order sampler is enough for the TTS task, and higher-order samplers do not make a significant difference.

\section{Experiments}
\subsection{Training Setup}

\paragraph{Data} We utilize the LJ-Speech dataset \citep{ljspeech17}, which contains $13,100$ samples, around $24$ hours in total, from a female speaker at a sampling rate of $22.05$ kHz. The test samples are extracted from both LJ-$001$ and LJ-$002$, and the remaining $12577$ samples are used for training. We follow the common practice, using the open-source tools \citep{g2pE2019} to convert the English grapheme sequence to phoneme sequence, and extracting the $80$-band mel-spectrogram with the FFT $1024$ points, $80$Hz and $7600$Hz lower and higher frequency cutoffs, and a hop length of $256$.

\paragraph{Model training} To conduct a fair comparison with diffusion models, we adopt the same network architecture and training settings used in Grad-TTS \citep{Grad-TTS}: 1) the encoder (\textit{i.e.}, text encoder and duration predictor) contains $7.2$M parameters and the U-Net based decoder contains $7.6$M parameters; 2) the model is trained with a batch size of $16$, and $1.7$M iterations in total on a single NVIDIA RTX $3090$, using $2.5$ days; 3) the Adam optimizer \citep{Adam} is employed with a constant learning rate of $0.0001$.
For noise schedules, we set $\beta_0 = 0.01, \beta_1 = 20$ for Bridge-VP (exactly the same as VP in SGMs) and $\beta_0 = 0.01, \beta_1 = 50$ for Bridge-gmax.

\paragraph{Evaluation} Following previous works \citep{Grad-TTS,Diffsinger,ProDiff}, we conduct the subjective tests MOS (Mean Opinion Score) and CMOS (Comparison Mean Opinion Score) to evaluate the overall subjective quality and comparison sample quality, respectively. To guarantee the reliability of the collected results, we use the open platform \textit{Amazon Mechanical Turk}, and require Master workers to complete the listening test. Specifically, the MOS scores of 20 test samples are given by 25 Master workers to evaluate the overall performance with a 5-point scale, where $1$ and $5$ denote the lowest (``Bad") and highest (``Excellent") quality respectively. The result is reported with a 95\% confidence interval. Each CMOS score is given by 15 Master workers to compare 20 test samples synthesized by two different models. More details of CMOS test can be visited in Appendix \ref{additionalresults}. In both MOS and CMOS tests, each of the test samples has been normalized for a fair comparison\footnote{\url{https://github.com/slhck/ffmpeg-normalize}}. To measure the inference speed, we calculate the real-time factor (RTF) on an NVIDIA RTX $3090$. 

\subsection{Results and Analyses}
We demonstrate the performance of Bridge-TTS on sample quality and inference speed separately, which guarantees a more precise comparison between multiple models. In Table \ref{mquality} and Table \ref{mspeed}, the test samples in LJ-Speech dataset are denoted as \textit{Recording}, the samples synthesized from ground-truth mel-spectrogram by vocoder are denoted as \textit{GT-Mel+voc.}, and the number of function evaluations is denoted as \textit{NFE}. We take the pre-trained HiFi-GAN \citep{HiFi-GAN}\footnote{\url{https://github.com/jik876/hifi-gan}} as the vocoder, aligned with other baseline settings. More details of baseline models are introduced in Appendix \ref{Baselinemodels}. 
In the sampling process of both tests, Grad-TTS employs ODE sampling and sets the prior distribution $p_T = \mathcal N(\boldsymbol z, \tau_d^{-1}\boldsymbol I)$ with a temperature parameter $\tau_d=1.5$. In Bridge-TTS, we use our first-order SDE sampler shown in~\eqref{eq:first-order-SDE} with a temperature parameter $\tau_b=2$ for the noise distribution $\epsilon = \mathcal N(\boldsymbol 0, \tau_b^{-1}\boldsymbol I)$, which is helpful to the TTS quality in our observation.

\noindent
\begin{minipage}[t]{\textwidth}
\small
\begin{minipage}[t]{0.47\textwidth}
\makeatletter\def\@captype{table}
\caption{The MOS comparison with 95\% confidence interval given numerous sampling steps.
}
\label{mquality}
\begin{center}
\begin{tabular}{l|ccc}
\toprule 
Model  & NFE & RTF ($\downarrow$) & MOS ($\uparrow$)  \\
\midrule
    Recording & / & / & 4.10 $\pm$ 0.06   \\
    GT-Mel + voc. & / & / & 3.93 $\pm$ 0.07 \\
\midrule
    FastSpeech 2 & 1 & 0.004 & 3.78 $\pm$ 0.07 \\
    VITS & 1 & 0.018 & 3.99 $\pm$ 0.07 \\
\midrule
    DiffSinger & 71 & 0.157 & 3.92 $\pm$ 0.06 \\
    ResGrad & 50 & 0.135 & 3.97 $\pm$ 0.07 \\
    Grad-TTS & 50 & 0.116 & 3.99 $\pm$ 0.07 \\
    \textbf{Ours (VP)} & 50 & 0.117 & \textbf{4.09 $\pm$ 0.07} \\
    \textbf{Ours (gmax)} & 50 & 0.117 & \textbf{4.07 $\pm$ 0.07} \\
\midrule
    Grad-TTS & 1000 & 2.233 & 3.98 $\pm$ 0.07 \\
    \textbf{Ours (VP)} & 1000 & 2.267 & \textbf{4.05 $\pm$ 0.07}  \\
    \textbf{Ours (gmax)} & 1000 & 2.267 & \textbf{4.07 $\pm$ 0.07}  \\
\bottomrule
\end{tabular}
\end{center}
\end{minipage}
\hspace{0.04\textwidth}
\begin{minipage}[t]{0.47\textwidth}
\makeatletter\def\@captype{table}
\caption{The MOS comparison with 95\% confidence interval in few-step generation.
}
\label{mspeed}
\small
\begin{center}
\vspace{1pt}
\begin{tabular}{l|ccc}
\toprule 
Model  & NFE & RTF ($\downarrow$) & MOS ($\uparrow$)  \\
\midrule
    Recording & / & / & 4.12 $\pm$ 0.06 \\
    GT-Mel + voc. & / & / & 4.01 $\pm$ 0.06 \\
\midrule
    FastSpeech 2 & 1 & 0.004 & 3.84 $\pm$ 0.07  \\
    CoMoSpeech & 1 & 0.007 & 3.74 $\pm$ 0.07 \\
\midrule    
    ProDiff & 2 & 0.019 & 3.67 $\pm$ 0.07 \\
    CoMoSpeech & 2 & 0.009 & 3.87 $\pm$ 0.07 \\
    \textbf{Ours (gmax)} & 2 & 0.009 & \textbf{4.04 $\pm$ 0.06} \\
\midrule
    DiffGAN-TTS & 4 & 0.014 & 3.78 $\pm$ 0.07\\
    Grad-TTS & 4 & 0.013 & 3.88 $\pm$ 0.07 \\
    FastGrad-TTS & 4 & 0.013 & 3.87 $\pm$ 0.07 \\
    ResGrad & 4 & 0.017 & 4.02 $\pm$ 0.06\\
    \textbf{Ours (gmax)} & 4 & 0.013 & \textbf{4.10 $\pm$ 0.06} \\    
\bottomrule
\end{tabular}
\end{center}
\end{minipage}
\end{minipage}

\paragraph{Generation quality}
Table \ref{mquality} compares the generation quality between Bridge-TTS and previous TTS systems. As shown, both Bridge-TTS models outperform three strong diffusion-based TTS systems: our diffusion counterpart Grad-TTS \citep{Grad-TTS}, the shallow diffusion model DiffSinger \citep{Diffsinger} and the residual diffusion model ResGrad \citep{ResGrad}. In comparison with the transformer-based model FastSpeech 2 \citep{Fastspeech2} and the end-to-end TTS system \citep{VITS}, we also exhibit stronger subjective quality. When NFE is either 1000 or 50, our Bridge-TTS achieves superior quality. One reason is that the condition information (\textit{i.e.}, text encoder output) in TTS synthesis is strong, and the other is that our first-order Bridger sampler maintains the sample quality when reducing the NFE.

\paragraph{Sampling speed}
Table \ref{mspeed} shows the evaluation of sampling speed with the Bridge-TTS-gmax model, as we observe that it achieves higher quality than the VP-based Bridge-TTS system. To conduct a fair comparison, we choose the NFE reported in the baseline models. As shown, in $4$-step sampling, we not only outperform our diffusion counterpart Grad-TTS \citep{Grad-TTS}, FastGrad-TTS \citep{FastGrad-TTS} using a first-order SDE sampler, and DiffGAN-TTS \citep{DiffGAN-TTS} by a large margin, but also achieve higher quality than ResGrad \citep{ResGrad} which stands on a pre-trained FastSpeech 2 \citep{Fastspeech2}. In 2-step sampling with a RTF of 0.009, we achieve higher quality than the state-of-the-art fast sampling method CoMoSpeech \citep{CoMoSpeech}. In comparison with 1-step method, FastSpeech 2 and CoMoSpeech, although our 2-step generation is slightly slower, we achieve distinctively better quality.

\subsection{Case Study}
We show a sample when NFE=4 in Figure \ref{fig:samplecomparison} (a), using our first-order ODE sampler shown in Eqn (\ref{eq:first-order-ODE}). As shown, Bridge-TTS clearly generates more details of the target than the diffusion counterpart Grad-TTS ($\tau_d=1.5$). Moreover, we show a 2-step ODE sampling trajectory of Bridge-TTS in Figure \ref{fig:samplecomparison} (b). As shown, with our data-to-data generation process, each sampling step is adding more details to refine the prior which has provided strong information about the target. 
More generated samples can be visited in Appendix \ref{generatedsamples}.

\begin{figure}[ht]
\centering
\includegraphics[width=\linewidth]{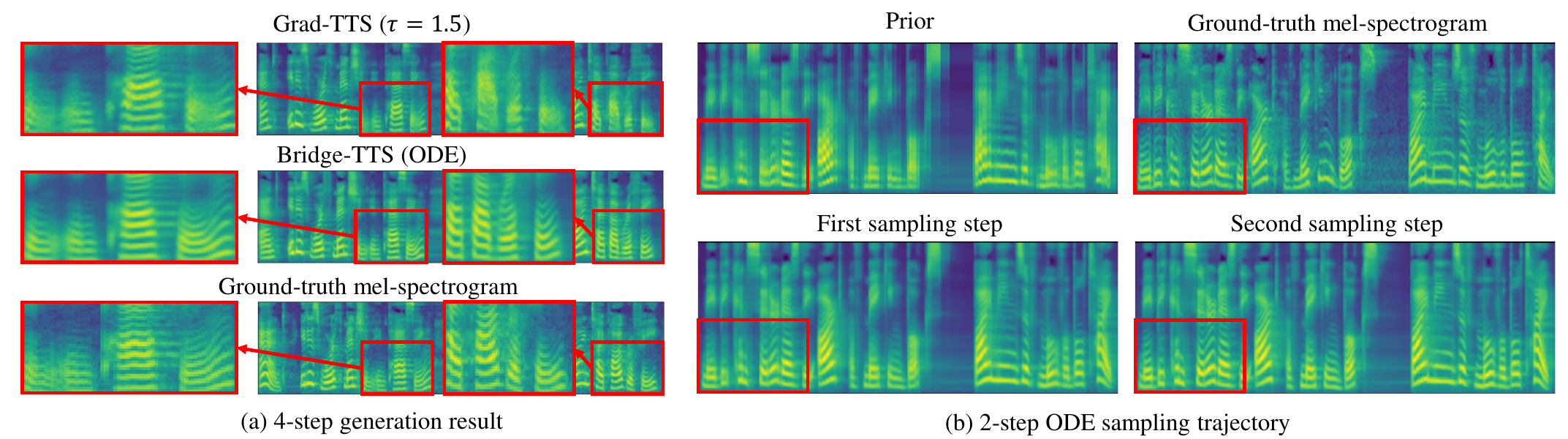}
\caption{We show a 4-step ODE generation result of Grad-TTS \citep{Grad-TTS} and Bridge-TTS in the left figure, and a 2-step ODE sampling trajectory of Bridge-TTS in the right one. The ground-truth mel-spectrogram is shown for comparison.}
\label{fig:samplecomparison}
\end{figure}

\subsection{Ablation Study}
\label{sec:ablation_study}
We conduct several comparison studies by showing the CMOS results between different designs of prior, noise schedule, and sampler when NFE$=$1000 and NFE$=$4. The base setup is the Bridge-gmax schedule, $\vx_0$ predictor, and temperature-scaled first-order SDE sampler ($\tau_b$ = 2). 

\paragraph{Prior}
We explore two training strategies that differ in their prior: $1$) like Grad-TTS \citep{Grad-TTS}, the encoder and decoder part are joint trained from scratch (i.e., mutable prior); $2$) the encoder is first trained with a warm-up stage and then the decoder is trained from scratch (i.e., fixed prior). 
\begin{wraptable}[11]{r}{0.43\textwidth}
\small
\vspace{-0.85\intextsep}
\caption{CMOS comparison of training and sampling settings of Bridge-TTS.}
\label{tab:methodcomparison}
\centering
\resizebox{0.43\textwidth}{!}{
\begin{tabular}{l|cc}
\toprule 
Method & NFE$=$4 & NFE$=$1000  \\
\midrule
Bridge-TTS (gmax) & 0 & 0 \\
\midrule
    \textit{w.} mutable prior & - 0.13 & - 0.17 \\
    \textit{w.} constant $g(t)$ & - 0.12 & - 0.14 \\
    \textit{w.} VP  & - 0.03   & - 0.08 \\
\midrule
    \textit{w.} SDE~($\tau_b$ = 1) & - 0.07 &  - 0.19 \\
    \textit{w.} ODE & - 0.10  & + 0.00 \\
\bottomrule
\end{tabular}
}
\end{wraptable}
It should be noted that in both strategies, the text encoder is trained with an equivalent objective. 
As shown, the latter consistently has better sample quality across different NFEs. Hence, we adopt it as our default setting.

\paragraph{Noise schedule}
We compare three different configurations for noise schedules: Bridge-gmax, Bridge-VP, and a simple schedule with $f(t) = 0, g(t) = 5$ that has virtually the same maximum marginal variance as Bridge-gmax, which we refer to as ``constant $g(t)$''. As shown in Table~\ref{tab:methodcomparison}, Bridge-gmax and Bridge-VP have overall similar performance, while the constant $g(t)$ has noticeably degraded quality than Bridge-gmax when NFE$=$1000.
Intuitively, the Bridge-gmax and Bridge-VP have an asymmetric pattern of marginal variance that assigns more steps for denoising, while the constant $g(t)$ yields a symmetric pattern. Empirically, such an asymmetric pattern of marginal variance helps improve sample quality. 
We provide a more detailed illustration of the noise schedules in Appendix~\ref{appendix:forwardprocess}. 

\paragraph{Sampling process}
For comparison between different sampling processes, the temperature-scaled SDE ($\tau_b$ = 2) achieves the best quality at both NFE$=$4 and NFE$=$1000. Compared with the vanilla SDE sampling (i.e., $\tau_b$ = 1), introducing the temperature sampling technique for SDE can effectively reduce artifacts in the background and enhance the sample quality when NFE is large, which is clearly reflected in the CMOS score in Table~\ref{tab:methodcomparison}. Meanwhile, the ODE sampler exhibits the same quality as the temperature-scaled SDE at NFE$=$1000, but it has more evident artifacts at NFE$=$4. 

\section{Related Work}
\label{relatedwork}
\paragraph{Diffusion-based TTS Synthesis}
Grad-TTS \citep{Grad-TTS} builds a strong TTS baseline with SGMs, surpassing the transformer-based \citep{Fastspeech1} and flow-based model \citep{kim2020glow}. In the following works, fast sampling methods are extensively studied, such as improving prior distribution \citep{PriorGrad}, designing training-free sampler \citep{Diff-TTS,FastGrad-TTS}, using auxiliary model \citep{Diffsinger,ResGrad}, introducing adversarial loss \citep{DiffGAN-TTS,SpecDiffGAN-TTS}, employing knowledge distillation \citep{ProDiff,CoMoSpeech}, developing lightweight U-Net \citep{LightGrad}, and leveraging CFM framework \citep{MachaTTS,VoiceFlow,ReFlow-TTS}. However, these methods usually explore to find a better trade-off between TTS quality and sampling speed than diffusion models instead of simultaneously improving both of them, and some of these methods require extra procedures, such as data pre-processing, auxiliary networks, and distillation stage, or prone to training instability. In contrast to each of the previous methods that study a data-to-noise process, we present a novel TTS system with a tractable Schrodinger bridge, demonstrating the advantages of the data-to-data process.

\paragraph{Schrodinger bridge}
Solving the Schrodinger bridge problem with an iterative procedure to simulate the intractable stochastic processes is widely studied \citep{Diffusion-SB-IPF,SB-DGM,solvingBridge,Diffusion-SB-Likelihood,IDBM,DSBM,ConstrainedBridges}. Two recent works \citep{I2SB,AlignedDSB} build the bridge in image translation and a biology task, while neither of them investigates the design space discussed in our work, which is of importance to sample quality and inference speed. 

\section{Conclusions}
We present Bridge-TTS, a novel TTS method built on data-to-data process, enabling mel-spectrogram generation from a deterministic prior via Schrodinger bridge. Under our 
theoretically elaborated tractable, flexible SB framework, we exhaustively explore the design space of noise schedule, model parameterization, and stochastic/deterministic samplers. Experimental results on sample quality and sampling efficiency in TTS synthesis demonstrate the effectiveness of our approach, which significantly outperforms previous methods and becomes a new baseline on this task.
We hope our work could open a new avenue for exploiting the board family of strong informative prior to further unleash the potential of generative models on a wide range of applications.

\bibliographystyle{plainnat}
\bibliography{refs}

\newpage
\setcounter{tocdepth}{2}
\begin{spacing}{1.4}
    {\small \tableofcontents}
\end{spacing}

\newpage
\appendix
\addcontentsline{toc}{section}{Appendix}

\section{Proofs}
\subsection{Tractable Schrodinger Bridge between Gaussian-Smoothed Paired Data}
\label{appendix:proof-tractable-sb}
\begin{proof}[Proof of Proposition~\ref{prop:tractable-sb}]
First, we conduct a similar transformation to~\cite{I2SB}, which reverses the forward-backward SDE system of the SB in~\eqref{eq:SB-SDE} and absorb the intractable term $\widehat\Psi,\Psi$ into the boundary condition. On one hand, by inspecting the backward SDE (\eqref{eq:SB-reverseSDE}) and its corresponding PDE (the second equation in~\eqref{eq:SB-PDE}), we can discover that if we regard $\widehat\Psi$ as a probability density function (up to a normalizing factor, which is canceled when we compute the score by operator $\nabla\log$), then the PDE of the backward SDE is a realization of the following forward SDE due to the Fokker-Plank equation~\citep{SGM}:
\begin{equation}
    \label{SB:linear-forwardSDE}
    \dm \x_t = \fv(\x_t, t)\dm t + g(t)\dm\wv_t,\quad\x_0 \sim \widehat{\Psi}_0,
\end{equation}
and its associated density of $\x_t$ is $\widehat\Psi_t$. When we assume $\fv(\x_t,t)=f(t)\x_t$ as a linear drift, then~\eqref{SB:linear-forwardSDE} becomes a 
\textit{narrow-sense linear SDE}, whose conditional distribution $\widehat\Psi_{t|0}(\x_t|\x_0)$ is a tractable Gaussian, which we will prove as follows. 

Specifically, Itô's formula~\citep{ito1951formula} tells us that, for a general SDE with drift $\mu_t$ and diffusion $\sigma_t$:
\begin{equation}
    \dm x_t=\mu_t(x_t)\dm t+\sigma_t(x_t)\dm w_t
\end{equation}
If $f(x,t)$ is a twice-differentiable function, then
\begin{equation}
    \dm f(x_t,t)=\left(\frac{\partial f}{\partial t}(x_t,t)+\mu_t(x_t)\frac{\partial f}{\partial x}(x_t,t)+\frac{\sigma_t^2(x_t)}{2}\frac{\partial^2f}{\partial x^2}(x_t,t)\right)\dm t+\sigma_t(x_t)\frac{\partial f}{\partial x}(x_t,t)\dm w_t
\end{equation}

Denote $\alpha_t=e^{\int_0^tf(\tau)\dm\tau}$, if we choose $\fv(\x,t)=\frac{\x}{\alpha_t}$, by Itô's formula we have
\begin{equation}
    \dm\left(\frac{\x_t}{\alpha_t}\right)=\frac{g(t)}{\alpha_t}\dm\wv_t
\end{equation}
which clearly leads to the result
\begin{equation}
    \frac{\x_t}{\alpha_t}-\frac{\x_0}{\alpha_0}\sim\Nc\left(\vect0,\int_0^t\frac{g^2(\tau)}{\alpha_\tau^2}\dm\tau\Iv\right)
\end{equation}
If we denote $\sigma_t^2=\int_0^t\frac{g^2(\tau)}{\alpha_\tau^2}\dm\tau$, finally we conclude that $\widehat\Psi_{t|0}(\x_t|\x_0)=\Nc(\alpha_t\x_0,\alpha_t^2\sigma_t^2\Iv)$.

On the other hand, due to the symmetry of the SB, we can reverse the time $t$ by $s=1-t$ and conduct similar derivations for $\Psi$, which finally leads to the result $\Psi_{t|1}(\x_t|\x_1)=\Nc(\bar\alpha_t\x_1,\alpha_t^2\bar\sigma_t^2\Iv)$.

Since we have Gaussian boundary conditions:
\begin{equation}
    p_{\data}=\widehat\Psi_0\Psi_0=\Nc(\x_0,\epsilon^2\Iv),\quad p_{\prior}=\widehat\Psi_1\Psi_1=\Nc(\x_1,\alpha_1^2\epsilon^2\Iv)
\end{equation}
Due to the properties of Gaussian distribution, it is intuitive to assume that the marginal distributions $\widehat\Psi_0,\Psi_1$ are also Gaussian. We parameterize them with undetermined mean and variance as follows:
\begin{equation}
    \widehat\Psi_0=\Nc(\av,\sigma^2\Iv),\quad \Psi_1=\Nc(\bv,\alpha_1^2\sigma^2\Iv)
\end{equation}
Since the conditional transitions $\widehat\Psi_{t|0},\Psi_{t|1}$ are known Gaussian as we have derived, the marginals at any $t\in [0,1]$ are also Gaussian (which can be seen as a simple linear Gaussian model):
\begin{equation}
\label{eq:marginal-t-undetermined}
    \widehat\Psi_t=\Nc(\alpha_t\av,(\alpha_t^2\sigma^2+\alpha_t^2\sigma_t^2)\Iv),\quad \Psi_t=\Nc(\bar\alpha_t\bv,(\alpha_t^2\sigma^2+\alpha_t^2\bar\sigma_t^2)\Iv)
\end{equation}
Then we can solve the coefficients $\av,\bv,\sigma$ by the boundary conditions. Note that $\bar\sigma_0^2=\sigma_1^2, \bar\alpha_0=\frac{1}{\alpha_1}$, and the product of two Gaussian probability density functions is given by
\begin{equation}
\Nc(\mu_1,\sigma_1^2)\Nc(\mu_2,\sigma_2^2)=\Nc\left(\frac{\sigma_2^2\mu_1+\sigma_1^2\mu_2}{\sigma_1^2+\sigma_2^2},\frac{\sigma_1^2\sigma_2^2}{\sigma_1^2+\sigma_2^2}\right)
\end{equation}
We have
\begin{align}
    &\begin{cases}
\widehat\Psi_0\Psi_0=\Nc(\av,\sigma^2\Iv)\Nc(\bar\alpha_0\bv,(\alpha_0^2\sigma^2+\alpha_0^2\bar\sigma_0^2)\Iv)=\Nc(\x_0,\epsilon^2\Iv)\\
\widehat\Psi_1\Psi_1=\Nc(\alpha_1\av,(\alpha_1^2\sigma^2+\alpha_1^2\sigma_1^2)\Iv)\Nc(\bv,\alpha_1^2\sigma^2\Iv)=\Nc(\x_1,\alpha_1^2\epsilon^2\Iv)
    \end{cases}\\
    \Rightarrow&\begin{cases}
        \displaystyle\frac{(\sigma^2+\sigma_1^2)\av+\sigma^2\frac{\bv}{\alpha_1}}{2\sigma^2+\sigma_1^2}=\x_0\\
        \displaystyle\frac{\alpha_1\sigma^2\av+(\sigma^2+\sigma_1^2)\bv}{2\sigma^2+\sigma_1^2}=\x_1\\
        \displaystyle\frac{\sigma^2(\sigma^2+\sigma_1^2)}{2\sigma^2+\sigma_1^2}=\epsilon^2
    \end{cases}
    \Rightarrow\begin{cases}
        \displaystyle\av=\x_0+\frac{\sigma^2}{\sigma_1^2}\left(\x_0-\frac{\x_1}{\alpha_1}\right)\\
        \displaystyle\bv=\x_1+\frac{\sigma^2}{\sigma_1^2}(\x_1-\alpha_1\x_0)\\
        \displaystyle\sigma^2=\epsilon^2+\frac{\sqrt{\sigma_1^4+4\epsilon^4}-\sigma_1^2}{2}
    \end{cases}
\end{align}

The proof is then completed by substituting these solved coefficients back into~\eqref{eq:marginal-t-undetermined}.
\end{proof}

\subsection{Bridge Sampling}
\label{appendix:proof-sb-sampling}
First of all, we would like to give some background information about exponential integrators~\citep{calvo2006class,hochbruck2009exponential}, which are widely used in recent works concerning fast sampling of diffusion ODE/SDEs~\citep{DPM-Solver,DPM-Solver++,zheng2023dpm,SEEDS}. Suppose we have an SDE (or equivalently an ODE by setting $g(t)=0$):
\begin{equation}
\label{eq:general_SDE}
    \dm\x_t=[a(t)\x_t+b(t)\Fv_\theta(\x_t,t)]\dm t+g(t)\dm\wv_t
\end{equation}
where $\Fv_\theta$ is the parameterized prediction function that we want to approximate with Taylor expansion. The usual way of representing its analytic solution $\x_t$ at time $t$ with respect to an initial
condition $\x_s$ at time $s$ is
\begin{equation}
\label{eq:general_SDE_solution_usual}
\x_t=\x_s+\int_s^t[a(\tau)\x_\tau+b(\tau)\Fv_\theta(\x_\tau,\tau)]\dm\tau+\int_s^t g(\tau)\dm\wv_\tau
\end{equation}
By approximating the involved integrals in~\eqref{eq:general_SDE_solution_usual}, we can obtain direct discretizations of~\eqref{eq:general_SDE} such as Euler's method. The key insight of exponential integrators is that, it is often better to utilize the ``semi-linear'' structure of~\eqref{eq:general_SDE} and analytically cancel the linear term $a(t)\x_t$. This way, we can obtain solutions that only involve integrals of $\Fv_\theta$ and result in lower discretization errors. Specifically, by the ``variation-of-constants'' formula, the exact solution of~\eqref{eq:general_SDE} can be alternatively given by
\begin{equation}
\x_t=e^{\int_s^ta(\tau)\dm\tau}\x_s+\int_s^te^{\int_\tau^ta(r)\dm r}b(\tau)\Fv_\theta(\x_\tau,\tau)\dm\tau+\int_s^te^{\int_\tau^ta(r)\dm r}g(\tau)\dm\wv_\tau
\end{equation}
or equivalently (assume $t<s$)
\begin{equation}
\label{eq:general_SDE_solution_exponential}
\x_t=e^{\int_s^ta(\tau)\dm\tau}\x_s+\int_s^te^{\int_\tau^ta(r)\dm r}b(\tau)\Fv_\theta(\x_\tau,\tau)\dm\tau+\sqrt{-\int_s^te^{2\int_\tau^ta(r)\dm r}g^2(\tau)\dm\tau}\epsilonv,\quad \epsilonv\sim\Nc(\vect0,\Iv)
\end{equation}

Then we prove Proposition~\ref{prop:sampler} below.
\begin{proof}[Proof of Proposition~\ref{prop:sampler}]
First, we consider the bridge SDE in~\eqref{eq:bridge-SDE}. By collecting the linear terms w.r.t. $\x_t$, the bridge SDE can be rewritten as
\begin{equation}
    \dm\x_t=\left[\left(f(t)+\frac{g^2(t)}{\alpha_t^2\sigma_t^2}\right)\x_t-\frac{g^2(t)}{\alpha_t\sigma_t^2}\x_\theta(\x_t,t)\right]\dm t+g(t)\dm\wv_t
\end{equation}
By corresponding it to~\eqref{eq:general_SDE}, we have
\begin{equation}
    a(t)=f(t)+\frac{g^2(t)}{\alpha_t^2\sigma_t^2},\quad b(t)=-\frac{g^2(t)}{\alpha_t\sigma_t^2}
\end{equation}
The exponents in~\eqref{eq:general_SDE_solution_exponential} can be calculated as
\begin{equation}
    \int_s^ta(\tau)\dm\tau=\int_s^tf(\tau)\dm\tau+\int_s^t\frac{(\sigma_\tau^2)'}{\sigma_\tau^2}\dm\tau=\int_s^tf(\tau)\dm\tau+\log\frac{\sigma_t^2}{\sigma_s^2}
\end{equation}
Thus
\begin{equation}
    e^{\int_s^ta(\tau)\dm\tau}=\frac{\alpha_t\sigma_t^2}{\alpha_s\sigma_s^2},\quad e^{\int_{\tau}^ta(r)\dm r}=\frac{\alpha_t\sigma_t^2}{\alpha_\tau\sigma_\tau^2}
\end{equation}
Therefore, the exact solution in~\eqref{eq:general_SDE_solution_exponential} becomes
 \begin{equation}
 \label{eq:proof-SDE-1}
     \x_t=\frac{\alpha_t\sigma_t^2}{\alpha_s\sigma_s^2}\x_s-\alpha_t\sigma_t^2\int_s^t\frac{g^2(\tau)}{\alpha_\tau^2\sigma_\tau^4}\x_\theta(\x_\tau,\tau)\dm\tau+\alpha_t\sigma_t^2\sqrt{-\int_s^t\frac{g^2(\tau)}{\alpha_\tau^2\sigma_\tau^4}\dm\tau}\epsilonv,\quad\epsilonv\sim \Nc(\vect0,\Iv)
 \end{equation}
where
\begin{equation}
 \label{eq:proof-SDE-2}
     \int_s^t\frac{g^2(\tau)}{\alpha_\tau^2\sigma_\tau^4}\dm\tau=\int_s^t\frac{(\sigma_\tau^2)'}{\sigma_\tau^4}\dm\tau=\frac{1}{\sigma_s^2}-\frac{1}{\sigma_t^2}
\end{equation}

Substituting~\eqref{eq:proof-SDE-2} into~\eqref{eq:proof-SDE-1}, we obtain the exact solution in~\eqref{eq:exact-solution-SDE}. If we take the first-order approximation (i.e., $\x_\theta(\x_\tau,\tau)\approx\x_\theta(\x_s,s)$ for $\tau\in [t,s]$), then we obtain the first-order transition rule in~\eqref{eq:first-order-SDE}.

Then we consider the bridge ODE in~\eqref{eq:bridge-ODE}. By collecting the linear terms w.r.t. $\x_t$, the bridge ODE can be rewritten as
\begin{equation}
    \dm\x_t=\left[\left(f(t)-\frac{g^2(t)}{2\alpha_t^2\bar\sigma_t^2}+\frac{g^2(t)}{2\alpha_t^2\sigma_t^2}\right)\x_t+\frac{g^2(t)\bar\alpha_t}{2\alpha_t^2\bar\sigma_t^2}\x_1-\frac{g^2(t)}{2\alpha_t\sigma_t^2}\x_\theta(\x_t,t)\right]\dm t
\end{equation}
By corresponding it to~\eqref{eq:general_SDE}, we have
\begin{equation}
    a(t)=f(t)-\frac{g^2(t)}{2\alpha_t^2\bar\sigma_t^2}+\frac{g^2(t)}{2\alpha_t^2\sigma_t^2},\quad b_1(t)=\frac{g^2(t)\bar\alpha_t}{2\alpha_t^2\bar\sigma_t^2},\quad b_2(t)=-\frac{g^2(t)}{2\alpha_t\sigma_t^2}
\end{equation}
The exponents in~\eqref{eq:general_SDE_solution_exponential} can be calculated as
\begin{equation}
    \begin{aligned}
\int_s^ta(\tau)\dm\tau&=\int_s^tf(\tau)\dm\tau-\int_s^t\frac{g^2(\tau)}{2\alpha_\tau^2\bar\sigma_\tau^2}\dm\tau+\int_s^t\frac{g^2(\tau)}{2\alpha_\tau^2\sigma_\tau^2}\dm\tau\\
&=\int_s^tf(\tau)\dm\tau+\int_s^t\frac{(\bar\sigma_\tau^2)'}{2\bar\sigma_\tau^2}\dm\tau+\int_s^t\frac{(\sigma_\tau^2)'}{2\sigma_\tau^2}\dm\tau\\
&=\int_s^tf(\tau)\dm\tau+\frac{1}{2}\log\frac{\bar\sigma_t^2}{\bar\sigma_s^2}+\frac{1}{2}\log\frac{\sigma_t^2}{\sigma_s^2}
\end{aligned}
\end{equation}
Thus
\begin{equation}
e^{\int_s^ta(\tau)\dm\tau}=\frac{\alpha_t\sigma_t\bar\sigma_t}{\alpha_s\sigma_s\bar\sigma_s},\quad e^{\int_{\tau}^ta(r)\dm r}=\frac{\alpha_t\sigma_t\bar\sigma_t}{\alpha_\tau\sigma_\tau\bar\sigma_\tau}
\end{equation}

Therefore, the exact solution in~\eqref{eq:general_SDE_solution_exponential} becomes
\begin{equation}
\label{eq:proof-ODE-1}
     \x_t=\frac{\alpha_t\sigma_t\bar\sigma_t}{\alpha_s\sigma_s\bar\sigma_s}\x_s+\frac{\bar\alpha_t\sigma_t\bar\sigma_t}{2}\int_s^t\frac{g^2(\tau)}{\alpha_\tau^2\sigma_\tau\bar\sigma_\tau^3}\x_1\dm\tau-\frac{\alpha_t\sigma_t\bar\sigma_t}{2}\int_s^t\frac{g^2(\tau)}{\alpha_\tau^2\sigma_\tau^3\bar\sigma_\tau}\x_\theta(\x_\tau,\tau)\dm\tau
 \end{equation}
 Due the relation $\sigma_t^2+\bar\sigma_t^2=\sigma_1^2$, the integrals can be computed by the substitution $\theta_t=\arctan(\sigma_t/\bar\sigma_t)$
 \begin{equation}
 \label{eq:proof-ODE-2}
 \begin{aligned}
     \int_s^t\frac{g^2(\tau)}{\alpha_\tau^2\sigma_\tau\bar\sigma_\tau^3}\dm\tau&=\int_s^t\frac{(\sigma_\tau^2)'}{\sigma_\tau\bar\sigma_\tau^3}\dm\tau\\
     &=\int_{\theta_s}^{\theta_t}\frac{1}{\sigma_1^4\sin\theta\cos^3\theta}\dm(\sigma_1^2\sin^2\theta)\\
     &=\frac{2}{\sigma_1^2}\int_{\theta_s}^{\theta_t}\frac{1}{\cos^2\theta}\dm\theta\\
     &=\frac{2}{\sigma_1^2}(\tan\theta_t-\tan\theta_s)\\
     &=\frac{2}{\sigma_1^2}\left(\frac{\sigma_t}{\bar\sigma_t}-\frac{\sigma_s}{\bar\sigma_s}\right) 
 \end{aligned}
 \end{equation}
 and similarly
 \begin{equation}
     \label{eq:proof-ODE-3}
     \int_s^t\frac{g^2(\tau)}{\alpha_\tau^2\sigma_\tau^3\bar\sigma_\tau}\dm\tau=\frac{2}{\sigma_1^2}\left(\frac{\bar\sigma_s}{\sigma_s}-\frac{\bar\sigma_t}{\sigma_t}\right)
 \end{equation}
 Substituting~\eqref{eq:proof-ODE-2} and~\eqref{eq:proof-ODE-3} into~\eqref{eq:proof-ODE-1}, we obtain the exact solution in~\eqref{eq:exact-solution-ODE}. If we take the first-order approximation (i.e., $\x_\theta(\x_\tau,\tau)\approx\x_\theta(\x_s,s)$ for $\tau\in [t,s]$), then we obtain the first-order transition rule in~\eqref{eq:first-order-ODE}.
\end{proof}
\section{Relationship with Brownian Bridge, Posterior Sampling and DDIM}
\subsection{Schrodinger Bridge Problem and Brownian Bridge}
\label{appendix:brownian-bridge}
For any path measure $\mu$ on $[0,1]$, we have $\mu=\mu_{0,1}\mu_{|0,1}$, where $\mu_{0,1}$ denotes the joint distribution of $\mu_0,\mu_1$, and $\mu_{|0,1}$ denotes the conditional path measure on $(0,1)$ given boundaries $\x_0,\x_1$. A high-level perspective is that, using the decomposition formula for KL divergence $\kl{p}{p^{\text{ref}}}=\kl{p_{0,1}}{p^{\text{ref}}_{0,1}}+\kl{p_{|0,1}}{p^{\text{ref}}_{|0,1}}$~\citep{leonard2014some}, the SB problem in~\eqref{eq:SB-DEF} can be reduced to the \textit{static SB} problem~\citep{Diffusion-SB-IPF,DSBM,tong2023simulation,tong2023improving}:
\begin{equation}
\min_{p_{0,1}\in\mathcal{P}_2}\kl{p_{0,1}}{p^{\text{ref}}_{0,1}},\quad \textit{s.t.}\,\, p_0=p_{\data},p_1=p_{\prior}
\end{equation}
which is proved to be an \textit{entropy-regularized optimal transport} problem when $p^{\text{ref}}$ is defined by a scaled Brownian process $\dm\x_t=\sigma\dm\wv_t$. We can draw similar conclusions for the more general case of reference SDE in~\eqref{eq:forwardSDE} with linear drift $\fv(\x_t,t)=f(t)\x_t$. Specifically, the KL divergence between the joint distribution of boundaries is
\begin{equation}
\begin{aligned}
    \kl{p_{0,1}}{p_{0,1}^{\text{ref}}}&=-\E_{p_{0,1}}[\log p_{0,1}^{\text{ref}}]-\Hc(p_{0,1})\\
    &=-\E_{p_{0}}[\log p_{0}^{\text{ref}}]-\E_{p_{0,1}}[\log p_{1|0}^{\text{ref}}]-\Hc(p_{0,1})
\end{aligned}
\end{equation}
where $\Hc(\cdot)$ is the entropy. As we have proved in Appendix~\ref{appendix:proof-tractable-sb}, $p_{t|0}^{\text{ref}}(\x_t|\x_0)=\Nc(\alpha_t\x_0,\alpha_t^2\sigma_t^2\Iv)$, thus
\begin{equation}
    \log p_{1|0}^{\text{ref}}(\x_1|\x_0)=-\frac{\|\x_1-\alpha_1\x_0\|_2^2}{2\alpha_1^2\sigma_1^2}
\end{equation}
Since $\E_{p_{0}}[\log p_{0}^{\text{ref}}]=\E_{p_{\data}}[\log p_{\data}]$ is irrelevant to $p$, the static SB problem is equivalent to
\begin{equation}
\min_{p_{0,1}\in\mathcal{P}_2} \E_{p_{0,1}(\x_0,\x_1)}[\|\x_1-\alpha_1\x_0\|_2^2]-2\alpha_1^2\sigma_1^2\Hc(p_{0,1}),\quad \textit{s.t.}\,\, p_0=p_{\data},p_1=p_{\prior}
\end{equation}
Therefore, it is an entropy-regularized optimal transport problem when $\alpha_1=1$.

While the static SB problem is generally non-trivial, there exists application cases when we can skip it: when the coupling $p_{0,1}$ of $p_{\data}$ and $p_{\prior}$ is unique and has no room for further optimization. (1) When $p_{\data}$ is a Dirac delta distribution and $p_{\prior}$ is a usual distribution~\citep{I2SB}. In this case, the SB is half tractable, and only the bridge SDE holds. (2) When paired data are considered, i.e., the coupling of $p_{\data}$ and $p_{\prior}$ is mixtures of dual Dirac delta distributions. In this case, however, $\kl{p_{0,1}}{p_{0,1}^{\text{ref}}}=\infty$, and the SB problem will collapse. Still, we can ignore such singularity, so that the SB is fully tractable, and bridge ODE can be derived.

After the static SB problem is solved, we only need to minimize $\kl{p_{|0,1}}{p^{\text{ref}}_{|0,1}}$ in order to solve the original SB problem. In fact, since there is no constraints, such optimization directly leads to $p_{t|0,1}=p^{\text{ref}}_{t|0,1}$ for $t\in (0,1)$. When $p^{\text{ref}}$ is defined by a scaled Brownian process $\dm\x_t=\sigma\dm\wv_t$, $p^{\text{ref}}_{t|0,1}$ is the common Brownian bridge~\citep{SE-Bridge,tong2023simulation,tong2023improving}. When $p^{\text{ref}}$ is defined by the narrow-sense linear SDE $\dm\x_t=f(t)\x_t\dm t+g(t)\dm\wv_t$ which we considered, $p^{\text{ref}}_{t|0,1}$ can be seen as the generalized Brownian bridge with linear drift and time-varying volatility, and we can derive its formula as follows.

Similar to the derivations in Appendix~\ref{appendix:proof-tractable-sb}, the transition probability from time $s$ to time $t$ ($s<t$) following the reference SDE $\dm\x_t=f(t)\x_t\dm t+g(t)\dm\wv_t$ is
\begin{equation}
    p^{\text{ref}}_{t|s}(\x_t|\x_s)=\Nc(\x_t;\alpha_{t|s}\x_s,\alpha_{t|s}^2\sigma_{t|s}^2\Iv)
\end{equation}
where $\alpha_{t|s},\sigma_{t|s}$ are the corresponding coefficients to $\alpha_t,\sigma_t$, while modifying the lower limit of integrals from $0$ to $s$:
\begin{equation}
\alpha_{t|s}=e^{\int_s^tf(\tau)\dm\tau},\quad \sigma_{t|s}^2=\int_s^t\frac{g^2(\tau)}{\alpha^2_{\tau|s}}\dm\tau
\end{equation}
We can easily identify that $\alpha_{t|s},\sigma_{t|s}$ are related to $\alpha_t,\sigma_t$ by
\begin{equation}
    \alpha_{t|s}=\frac{\alpha_t}{\alpha_s},\quad \sigma_{t|s}^2=\alpha_s^2(\sigma_t^2-\sigma_s^2)
\end{equation}
Therefore
\begin{equation}
    p^{\text{ref}}_{t|s}(\x_t|\x_s)=\Nc\left(\x_t;\frac{\alpha_t}{\alpha_s}\x_s,\alpha_t^2(\sigma_t^2-\sigma_s^2)\Iv\right)
\end{equation}
Due to the Markov property of the SDE, we can compute $p^{\text{ref}}_{t|0,1}$ as
\begin{equation}
\label{eq:brownian-bridge-01-derivation}
\begin{aligned}
    p^{\text{ref}}_{t|0,1}(\x_t|\x_0,\x_1)&=\frac{p^{\text{ref}}_{t,1|0}(\x_t,\x_1|\x_0)}{p^{\text{ref}}_{1|0}(\x_1|\x_0)}\\
    &=\frac{p^{\text{ref}}_{t|0}(\x_t|\x_0)p^{\text{ref}}_{1|t}(\x_1|\x_t)}{p^{\text{ref}}_{1|0}(\x_1|\x_0)}\\
    &\propto\frac{\exp\left(-\frac{\|\x_t-\alpha_t\x_0\|_2^2}{2\alpha_t^2\sigma_t^2}\right)\exp\left(-\frac{\|\x_1-\frac{\alpha_1}{\alpha_t}\x_t\|_2^2}{2\alpha_1^2(\sigma_1^2-\sigma_t^2)}\right)}{\exp\left(-\frac{\|\x_1-\alpha_1\x_0\|_2^2}{2\alpha_1^2\sigma_1^2}\right)}\\
    &\propto\exp\left(-\frac{\|\x_t-\alpha_t\x_0\|_2^2}{2\alpha_t^2\sigma_t^2}-\frac{\|\x_t-\bar\alpha_t\x_1\|_2^2}{2\alpha_t^2\bar\sigma_t^2}\right)\\
    &\propto \exp\left(-\frac{\|\x_t-\frac{\alpha_t\bar\sigma_t^2\x_0+\bar\alpha_t\sigma_t^2\x_1}{\sigma_1^2}\|_2^2}{2\frac{\alpha_t^2\sigma_t^2\bar\sigma_t^2}{\sigma_1^2}}\right)
\end{aligned}
\end{equation}
Therefore, $p^{\text{ref}}_{t|0,1}=\Nc\left(\frac{\alpha_t\bar\sigma_t^2\x_0+\bar\alpha_t\sigma_t^2\x_1}{\sigma_1^2},\frac{\alpha_t^2\bar\sigma_t^2\sigma_t^2}{\sigma_1^2}\Iv\right)$, which equals the SB marginal in~\eqref{eq:tractable_marginal}.
\subsection{Posterior Sampling on a Brownian Bridge and DDIM}
\label{appendix:posterior-sampling}
\paragraph{Posterior Sampling and Bridge SDE} \cite{I2SB} proposes a method called \textit{posterior sampling} to sample from bridge: when $p^{\text{ref}}$ is defined by $\dm\x_t=\sqrt{\beta_t}\dm\wv_t$, we can sample $\x_{N-1},\dots,\x_{n+1},\x_n,\dots,\x_0$ at timesteps $t_{N-1},\dots,t_{n+1},t_n,\dots,t_0$ sequentially, where at each step the sample is generated from the DDPM posterior~\citep{DDPM}:
\begin{equation}
    p(\x_n | \x_0, \x_{n+1})= \mathcal N\left(\x_n; \frac{\alpha_n^2}{\alpha_n^2 + \sigma_n^2}\x_0 + \frac{\sigma_n^2}{\alpha_n^2 + \sigma_n^2}\x_{n+1}, \frac{\sigma_n^2\alpha_n^2}{\alpha_n^2 + \sigma_n^2} \Iv\right),
\end{equation}
where $\alpha_n^2 = \int_{t_{n}}^{t_{n+1}} \beta(\tau)\dm\tau $ is the accumulated noise between two timesteps $(t_n, t_{n+1})$, $\sigma_n^2=\int_0^{t_n}\beta(\tau)\dm\tau$, and $\x_0$ is predicted by the network.

While they only consider $f(t)=0$ and prove the case for discrete timesteps by onerous mathematical induction, such posterior is essentially a ``shortened'' Brownian bridge. Suppose we already draw a sample $\x_s\sim p^{\text{ref}}_{s|0,1}$, then the sample at time $t<s$ can be drawn from $p^{\text{ref}}_{t|0,1,s}$, which equals $p^{\text{ref}}_{t|0,s}$ due to the Markov property of the SDE. Similar to the derivation in~\eqref{eq:brownian-bridge-01-derivation}, such shortened Brownian bridge is
\begin{equation}
    p^{\text{ref}}_{t|0,s}(\x_t|\x_0,\x_s)=\Nc\left(\x_t;\frac{\alpha_t(\sigma_s^2-\sigma_t^2)\x_0+\frac{\alpha_t}{\alpha_s}\sigma_t^2\x_s}{\sigma_s^2},\frac{\alpha_t^2\sigma_t^2(\sigma_s^2-\sigma_t^2)}{\sigma_s^2}\Iv\right)
\end{equation}
which is exactly the same as the first-order discretization of bridge SDE in~\eqref{eq:first-order-SDE} when $\x_0$ is predicted by the network $\x_\theta(\x_s,s)$.
\paragraph{DDIM and Bridge ODE} DDIM~\citep{DDIM} is a sampling method for diffusion models, whose deterministic case is later proved to be the first-order discretization of certain solution forms of the diffusion ODE~\citep{DPM-Solver,DPM-Solver++}. Under our notations of $\alpha_t,\sigma_t^2$, the update rule of DDIM is~\citep{DPM-Solver++}
\begin{equation}
\label{eq:ddim}
    \x_t=\frac{\alpha_t\sigma_t}{\alpha_s\sigma_s}\x_s+\alpha_t\left(1-\frac{\sigma_t^2}{\sigma_s^2}\right)\x_\theta(\x_s,s)
\end{equation}
In the limit of $\frac{\sigma_s}{\sigma_1},\frac{\sigma_t}{\sigma_1}\rightarrow 0$, we have $\frac{\bar\sigma_s}{\sigma_1},\frac{\bar\sigma_t}{\sigma_1}\rightarrow 1$. Therefore, $\frac{\bar\sigma_t}{\bar\sigma_s}\rightarrow 1$, and we can discover that~\eqref{eq:first-order-ODE} reduces to~\eqref{eq:ddim}.
\begin{corollary}[1-step First-Order Bridge SDE/ODE Sampler Recovers Direct Data Prediction]
When $s=1$ and $t=0$, the first-order discretization of bridge SDE/ODE is
\begin{equation}
\label{one-stepEq}
    \x_0=\x_\theta(\x_1,1)
\end{equation}
\end{corollary}

\section{High-Order Samplers}
\label{appendix:high-order-sampler}
We can develop high-order samplers by approximating $\x_\theta(\x_\tau,\tau),\tau\in[t,s]$ with high-order Taylor expansions. Specifically, we take the second-order case of the bridge SDE as an example. For the integral $\int_s^t\frac{g^2(\tau)}{\alpha_\tau^2\sigma_\tau^4}\x_\theta(\x_\tau,\tau)\dm\tau$ in~\eqref{eq:exact-solution-SDE}, we can use the change-of-variable $\lambda_t=-\frac{1}{\sigma_t^2}$. Since $(\lambda_t)'=\frac{g^2(t)}{\alpha_t^2\sigma_t^4}$, the integral becomes
\begin{equation}
\label{eq:second-order-1}
\begin{aligned}
\int_s^t\frac{g^2(\tau)}{\alpha_\tau^2\sigma_\tau^4}\x_\theta(\x_\tau,\tau)\dm\tau&=\int_{\lambda_s}^{\lambda_t}\x_\theta(\x_{\tau_\lambda},\tau_\lambda)\dm\lambda\\
&\approx \int_{\lambda_s}^{\lambda_t}\x_\theta(\x_s,s)+(\lambda-\lambda_s)\x^{(1)}_\theta(\x_s,s)\dm\lambda\\
&=(\lambda_t-\lambda_s)\x_\theta(\x_s,s)+\frac{(\lambda_t-\lambda_s)^2}{2}\x^{(1)}_\theta(\x_s,s)
\end{aligned}
\end{equation}
where $\tau_\lambda$ is the inverse mapping of $\lambda_\tau$, $\x^{(1)}_\theta$ is the first-order derivative of $\x_\theta$ w.r.t $\lambda$, and we have used the second-order Taylor expansion $\x_\theta(\x_{\tau_\lambda},\tau_\lambda)\approx \x_\theta(\x_s,s)+(\lambda-\lambda_s)\x^{(1)}_\theta(\x_s,s)$. $\x^{(1)}_\theta$ can be estimated by finite difference, and a simple treatment is the predictor-corrector method. We first compute $\hat\x_t$ by the first-order update rule in~\eqref{eq:first-order-SDE}, which is used to estimate $\x^{(1)}_\theta(\x_s,s)$:
$
    \x^{(1)}_\theta(\x_s,s)\approx \frac{\x_\theta(\hat\x_t,t)-\x_\theta(\x_s,s)}{\lambda_t-\lambda_s}
$. Substituting it into~\eqref{eq:second-order-1}, we have
$
    \int_s^t\frac{g^2(\tau)}{\alpha_\tau^2\sigma_\tau^4}\x_\theta(\x_\tau,\tau)\dm\tau\approx (\lambda_t-\lambda_s)\frac{\x_\theta(\x_s,s)+\x_\theta(\hat\x_t,t)}{2}
$
which literally can be seen as replacing $\x_\theta(\x_s,s)$ in~\eqref{eq:first-order-SDE} with $\frac{\x_\theta(\x_s,s)+\x_\theta(\hat\x_t,t)}{2}$. Similar derivations can be done for the bridge ODE. We summarize the second-order samplers in Algorithm~\ref{algorithm:SDE} and Algorithm~\ref{algorithm:ODE}.

\begin{algorithm}[H]
\caption{Second-order sampler for the bridge SDE}
\label{algorithm:SDE}
\textbf{Input:} Number of function evaluations (NFE) $2N$, timesteps $1=t_{N}>t_{N-1}>\dots>t_{n}>t_{n-1}>\dots>t_0=0$, initial condition $\x_1$
\begin{algorithmic}[1]
\FOR {$n = N$ \TO $1$}
\STATE $s\leftarrow t_n$
\STATE $t\leftarrow t_{n-1}$
\STATE Prediction: 
$
     \hat\x_t\leftarrow\frac{\alpha_t\sigma_t^2}{\alpha_s\sigma_s^2}\x_s+\alpha_t\left(1-\frac{\sigma_t^2}{\sigma_s^2}\right)\x_\theta(\x_s,s)+\alpha_t\sigma_t\sqrt{1-\frac{\sigma_t^2}{\sigma_s^2}}\epsilonv,\quad \epsilonv\sim\Nc(\vect 0,\Iv)
$
\STATE Correction: 
$
     \x_t\leftarrow\frac{\alpha_t\sigma_t^2}{\alpha_s\sigma_s^2}\x_s+\alpha_t\left(1-\frac{\sigma_t^2}{\sigma_s^2}\right)\frac{\x_\theta(\x_s,s)+\x_\theta(\hat\x_t,t)}{2}+\alpha_t\sigma_t\sqrt{1-\frac{\sigma_t^2}{\sigma_s^2}}\epsilonv,\quad \epsilonv\sim\Nc(\vect 0,\Iv)
$
\ENDFOR
\end{algorithmic}
\textbf{Output:} $\x_{0}$
\end{algorithm}
\begin{algorithm}[H]
\caption{Second-order sampler for the bridge ODE}
\label{algorithm:ODE}
\textbf{Input:} Number of function evaluations (NFE) $2N$, timesteps $1=t_{N}>t_{N-1}>\dots>t_{n}>t_{n-1}>\dots>t_0=0$, initial condition $\x_1$
\begin{algorithmic}[1]
\FOR {$n = N$ \TO $1$}
\STATE $s\leftarrow t_n$
\STATE $t\leftarrow t_{n-1}$
\STATE Prediction: 
$
\hat\x_t\leftarrow\frac{\alpha_t\sigma_t\bar\sigma_t}{\alpha_s\sigma_s\bar\sigma_s}\x_s+\frac{\alpha_t}{\sigma_1^2}\left[\left(\bar\sigma_t^2-\frac{\bar\sigma_s\sigma_t\bar\sigma_t}{\sigma_s}\right)\x_\theta(\x_s,s)+\left(\sigma_t^2-\frac{\sigma_s\sigma_t\bar\sigma_t}{\bar\sigma_s}\right)\frac{\x_1}{\alpha_1}\right]
$
\STATE Correction: 
$
     \x_t\leftarrow\frac{\alpha_t\sigma_t\bar\sigma_t}{\alpha_s\sigma_s\bar\sigma_s}\x_s+\frac{\alpha_t}{\sigma_1^2}\left[\left(\bar\sigma_t^2-\frac{\bar\sigma_s\sigma_t\bar\sigma_t}{\sigma_s}\right)\frac{\x_\theta(\x_s,s)+\x_\theta(\hat\x_t,t)}{2}+\left(\sigma_t^2-\frac{\sigma_s\sigma_t\bar\sigma_t}{\bar\sigma_s}\right)\frac{\x_1}{\alpha_1}\right]
$
\ENDFOR
\end{algorithmic}
\textbf{Output:} $\x_{0}$
\end{algorithm}


\section{Model Parameterization}
\label{appendix:training-objective}

Apart from $\x_0$ predictor $\x_\theta$ presented in Section~\ref{sec:model_training}, we can consider other parameterizations:
\begin{itemize}
    \item Noise predictor $\epsilonv_\theta^{\widehat\Psi}$ corresponding to $\nabla\log \widehat\Psi_t=-\frac{\x_t-\alpha_t\x_0}{\alpha_t^2\sigma_t^2}$ that used in I2SB~\citep{I2SB}. The prediction target of $\epsilonv_\theta^{\widehat\Psi}$ is:
    \begin{equation}
    \label{eq:noise-predictor}
    \epsilonv_\theta^{\widehat\Psi} \rightarrow \frac{\x_t-\alpha_t\x_0}{\alpha_t\sigma_t}
    \end{equation}
    \item Noise predictor $\epsilonv_\theta^{\SB}$ corresponding to the score $\nabla\log p_t$ of the SB. Since
    $
        \nabla\log p_t(\x_t)=-\frac{\x_t-\frac{\alpha_t\bar\sigma_t^2\x_0+\bar\alpha_t\sigma_t^2\x_1}{\sigma_1^2}}{\frac{\alpha_t^2\bar\sigma_t^2\sigma_t^2}{\sigma_1^2}}
    $, the prediction target of $\epsilonv_\theta^{\SB}$ is
    \begin{equation}
    \epsilonv_\theta^{\SB}\rightarrow \frac{\x_t-\frac{\alpha_t\bar\sigma_t^2\x_0+\bar\alpha_t\sigma_t^2\x_1}{\sigma_1^2}}{\frac{\alpha_t\bar\sigma_t\sigma_t}{\sigma_1}}
    \end{equation}
    \item Velocity predictor $\vv_\theta$ arising from flow matching techniques~\citep{FlowMatching,tong2023improving,tong2023simulation,zheng2023improved}, which aims to directly predict the drift of the PF-ODE:
    \begin{equation}
        \vv_\theta\rightarrow f(t)\x_t- \frac{1}{2}g^2(t) \frac{\x_t-\bar\alpha_t\x_1}{\alpha_t^2\bar\sigma_t^2}+\frac{1}{2}g^2(t)\frac{\x_t-\alpha_t\x_0}{\alpha_t^2\sigma_t^2}
    \end{equation}
\end{itemize}
Empirically, across all parameterizations, we observe that the $\x_0$ predictor and the noise predictor $\epsilonv_\theta^{\widehat\Psi}$ work well in the TTS task and Table~\ref{tab:parameterizationcomparison} shows that the $\x_0$ predictor is generally better in sample quality. Hence, we adopt the $\x_0$ predictor as the default training setup for Bridge-TTS. For the $\epsilonv_\theta^{\SB}$ predictor and $\vv_\theta$ predictor, we find that they lead to poor performance on the TTS task.
We can intuitively explain this phenomenon by taking a simple case $f(t)=0,g(t)=\sigma$. In this case, we have $\x_t=(1-t)\x_0+t\x_1+\sigma\sqrt{t(1-t)}\epsilonv,\epsilonv\sim\Nc(\vect0,\Iv)$, and the prediction targets are
\begin{equation}
    \begin{aligned}
        \x_\theta&\rightarrow\x_0\\
        \epsilonv_\theta^{\widehat\Psi}&\rightarrow\frac{\x_t-\x_0}{\sigma\sqrt{t}}=\sqrt{t}(\x_1-\x_0)+\sigma\sqrt{1-t}\epsilonv\\
        \epsilonv_\theta^{\SB}&\rightarrow\frac{\x_t-(1-t)\x_0-t\x_1}{\sigma\sqrt{t(1-t)}}=\epsilonv\\
        \sqrt{t(1-t)}\vv_\theta&\rightarrow\frac{(1-2t)\x_t-(1-t)\x_0+t\x_1}{2\sqrt{t(1-t)}}=\sqrt{t(1-t)}(\x_1-\x_0)+\sigma\frac{1-2t}{2}\epsilonv
    \end{aligned}
\end{equation}
Therefore, $\epsilonv_\theta^{\SB}$ and $\vv_\theta$ both predict $\epsilonv$ when $t\rightarrow 1$, while $\x_\theta$ and $\epsilonv_\theta^{\widehat\Psi}$ tends to predict $\x_0,\x_1$-related terms in such scenario. We can conclude that the former way of prediction is harmful on TTS task.

\begin{table}[ht]
\small
\caption{CMOS comparison of different parameterizations of Bridge-TTS.}
\label{tab:parameterizationcomparison}
\centering
\begin{tabular}{l|cc}
\toprule 
Method & NFE$=$4 & NFE$=$1000  \\
\midrule
Bridge-TTS (gmax + $\x_0$ predictor) & 0 & 0 \\
Bridge-TTS (gmax + $\epsilonv_\theta^{\widehat\Psi}$ predictor) & - 0.15 & - 0.12  \\
\bottomrule
\end{tabular}
\end{table}

\section{Forward Process}
\label{appendix:forwardprocess}
In this section, we display the stochastic trajectory of the Bridge-SDE in \eqref{eq:bridge-SDE} and compare it with the diffusion counterpart in \eqref{eq:forwardSDE-GradTTS}. In general, the marginal distribution of these SDEs shares the form $p_t = \gN(\vx_t ; \vw_{t} \xv_0 + \bar\vw_t \xv_1,  \tilde\sigma_t^2 \boldsymbol I)$. In Figure~\ref{fig:forwardweight}, we show the scaling factors $\vw_t$ and $\bar\vw_t$ for $\vx_0$ and $\vx_1$ and the variance $\tilde\sigma_t^2$ at time $t$. As described in Section~\ref{sec:ablation_study}, the Bridge-gmax and Bridge-VP have an asymmetric pattern of marginal variance that uses more steps to denoise towards the ground truth $\x_0$, while the constant $g(t)$ schedule specifies the same noise-additive and denoising steps. As a comparison, the diffusion-based model only performs denoising steps.

\begin{figure}[htbp]
    \centering
    \includegraphics[width=0.95\textwidth]{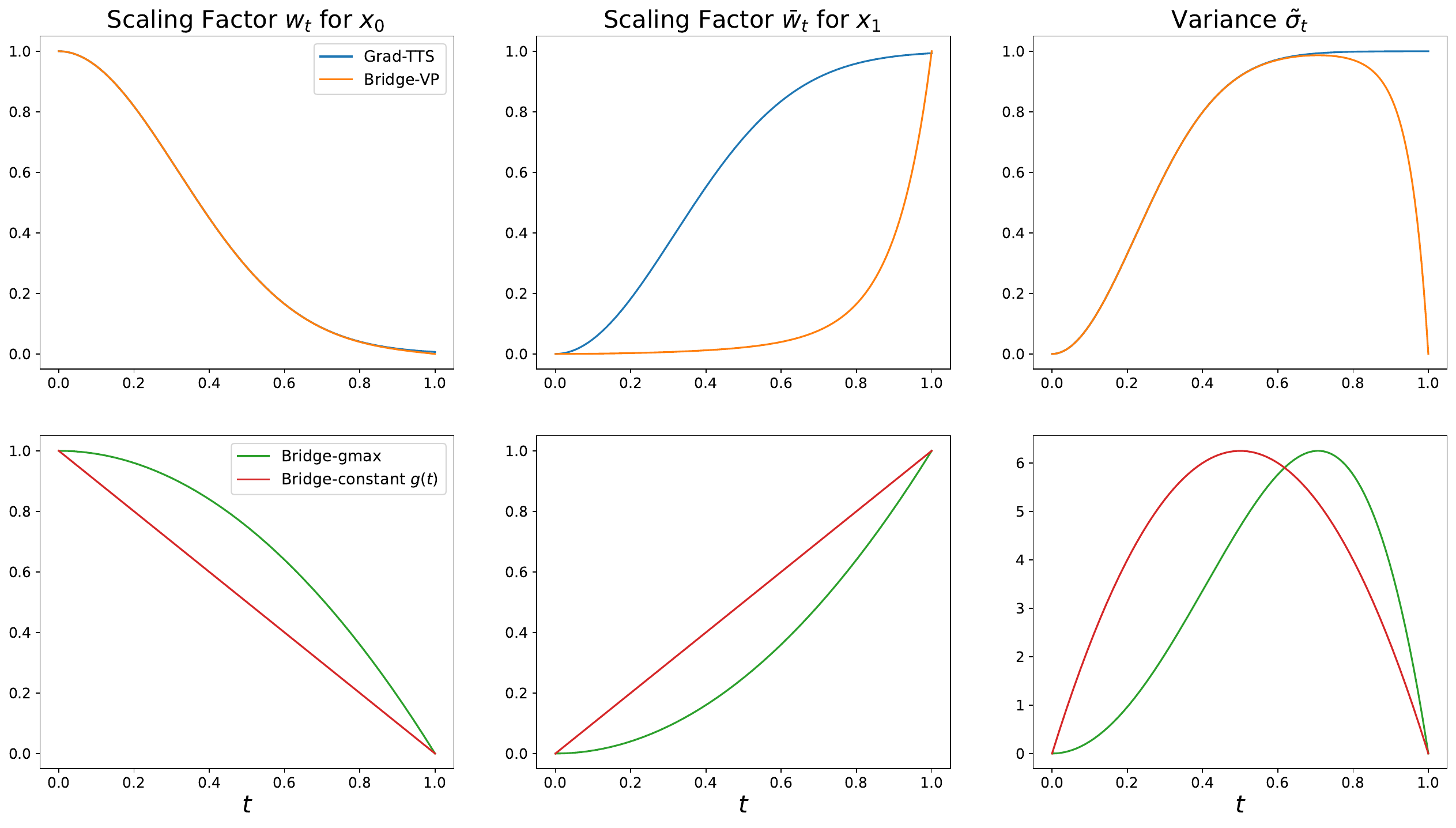}
    \caption{The scaling factor and variance in Grad-TTS and Bridge-TTS.}
    \label{fig:forwardweight}
\end{figure}

\newpage

\section{Baseline Models}
\label{Baselinemodels}

Apart from ground-truth recording and the sample synthesized by vocoder from ground-truth mel-spectrogram, we take seven diffusion-based TTS systems, one end-to-end TTS system, and one transformer-based TTS model as our baseline models. We follow their official implementation or the settings reported in their publications to produce the results. We introduce each of our baseline models below:

\textbf{1. FastSpeech 2} \citep{Fastspeech2} 
is one of the most popular non-autoregressive TTS models, and widely used as the baseline in previous diffusion-based TTS systems \citep{ResGrad,Diffsinger,CoMoSpeech}. Following its original setting, we train the model with a batch size of 48 sentences and 160k training steps until convergence by using 8 NVIDIA V100 GPU.

\textbf{2. VITS} \citep{VITS} 
provides a strong baseline of end-to-end TTS systems and is widely taken as a baseline in TTS systems for sample quality comparison. Different from other baseline models using pre-trained vocoder to generate waveform, VITS directly synthesizes waveform from text input. In training and testing, we follow their open-source implementation\footnote{\url{https://github.com/jaywalnut310/vits}}.

\textbf{3. DiffSinger} \citep{Diffsinger} 
is a TTS model developed for TTS synthesis and text-to-singing synthesis. It is built on denoising diffusion probabilistic models \citep{DDPM}, using standard Gaussian noise $\mathcal{N}(\boldsymbol 0, \boldsymbol I)$ in the diffusion process. Moreover, an auxiliary model is trained to enable its shallow reverse process, \textit{i.e.}, reducing the distance between prior distribution and data distribution. We follow their open-source implementation\footnote{\url{https://github.com/MoonInTheRiver/DiffSinger}}, which contains a warm-up stage for auxiliary model training and a main stage for diffusion model training.

\textbf{4. DiffGAN-TTS} \citep{DiffGAN-TTS}\footnote{\url{https://github.com/keonlee9420/DiffGAN-TTS}} develops expressive generator and time-dependent discriminator to learn the non-Gaussian denoising distribution \citep{DiffusionGAN} in few-step sampling process of diffusion models. Following their publication, we train DiffGAN-TTS with time steps $T=4$. For both the generator and the discriminator, we use the Adam optimizer, with $\beta_{1}=0.5$ and $\beta_{2}=0.9$. Models are trained using a single NVIDIA V100 GPU. We set the batch size as 32, and train models for 400k steps until loss converges.

\textbf{5. ProDiff} \citep{ProDiff} 
is a fast TTS model using progressive distillation \citep{ProgressiveDis}. The standard Gaussian noise $\mathcal{N}(\boldsymbol 0, \boldsymbol I)$ is used in the diffusion process and taken as the prior distribution. We use their $2$-step diffusion-based student model, which is distilled from a 4-step diffusion-based teacher model ($\boldsymbol x_0$ prediction). We follow their open-source implementation\footnote{\url{https://github.com/Rongjiehuang/ProDiff}}.

\textbf{6. Grad-TTS} \citep{Grad-TTS}\footnote{\url{https://github.com/huawei-noah/Speech-Backbones/tree/main/Grad-TTS}} 
is a widely used baseline in diffusion models \citep{ProDiff,ResGrad,LightGrad,CoMoSpeech} and conditional flow matching \citep{MachaTTS,VoiceFlow} based TTS systems. It is established on SGMs, providing a strong baseline of generation quality. Moreover, it realizes fast sampling with the improved prior distribution $\mathcal{N}(\boldsymbol \mu, \boldsymbol I)$ and the temperature parameter $\tau=1.5$ in inference. Following its original setting and publicly available implementation, we train the model with a batch size of 16 and 1.7 million steps on 1 NVIDIA 2080 GPU. The Adam optimizer is used and the learning rate is set to a constant, 0.0001.

\textbf{7. FastGrad-TTS} \citep{FastGrad-TTS} 
equips pre-trained Grad-TTS \citep{Grad-TTS} with the first-order SDE sampler proposed by \citep{GradVC}. The Maximum Likelihood solver reduces the mismatch between the reverse and the forward process. In comparison with the first-order Euler scheme, this solver has shown improved quality in both voice conversion and TTS synthesis. We implement it for the pre-trained Grad-TTS model with the Equation (6)-(9) in its publication.

\textbf{8. ResGrad} \citep{ResGrad} is a diffusion-based post-processing module to improve the TTS sample quality, where the residual information of a pre-trained FastSpeech 2 \citep{Fastspeech2} model is generated by a diffusion model. The standard Gaussian noise $\mathcal{N}(\boldsymbol 0, \boldsymbol I)$ is used in the diffusion process and taken as prior. We invite the authors to generate some test samples for us.

\textbf{9. CoMoSpeech} \citep{CoMoSpeech}\footnote{\url{https://github.com/zhenye234/CoMoSpeech}} 
is a recent fast sampling method in TTS and text-to-singing synthesis, achieving one-step generation with the distillation technique in consistency models \citep{ConsistencyModels}. As Grad-TTS is employed as its TTS backbone, the model uses $\mathcal{N}(\boldsymbol \mu, \boldsymbol I)$ as prior distribution and is trained for 1.7 million iterations on a single NVIDIA A100 GPU with a batch size of 16. The Adam optimizer is adopted with a learning rate 0.0001.

\section{Additional Results}
\label{additionalresults}

\subsection{CMOS Test}

\begin{table}[h]
\small
    \centering
    \caption{CMOS criteria.}
    \label{tab:cmos}
    \begin{tabular}{c|c|c|c|c|c|c|c}
    \toprule
    Comparison Quality  & \multicolumn{3}{c|}{Left better} & Equal & \multicolumn{3}{c}{Right better} \\
    \midrule
    Rating    &  +3 & +2 & +1 & 0 & -1 & -2 & -3 \\
    \bottomrule
    \end{tabular}
\end{table}

We conduct the Comparison Mean Opinion Score (CMOS) test by 15 Master workers on Amazon Mechanical Turk to compare the generation quality of two different models. The raters give a score from +3 (left better) to -3 (right better) with 1 point increments on each of the $20$ pair samples generated by two models, as shown in Table \ref{tab:cmos}. The three scores 3, 2, 1 denotes much better, better, and slightly better, respectively. The plus sign $+$ and the minus sign $-$ denotes the left and the right model respectively.

\begin{table}[ht]
\small
\caption{CMOS comparison between Grad-TTS and Bridge-TTS.}
\label{tab:cmosqualitycomparison}
\centering
\begin{tabular}{l|c}
\toprule 
Method & CMOS ($\uparrow$)  \\
\midrule
Grad-TTS & 0  \\
Bridge-TTS (gmax) & +0.21  \\
\bottomrule
\end{tabular}
\end{table}

To further demonstrate the advantage of data-to-data process in Bridge-TTS over the data-to-noise process in Grad-TTS, we conduct a CMOS test between Bridge-TTS (Bridge-gmax schedule with $\beta_0=0.01, \beta_1=50$, $\x_0$ predictor, and first-order SDE sampler with $\tau_b=2$) and Grad-TTS ($\tau_d=1.5$). As shown in Table \ref{tab:cmosqualitycomparison}, our Bridge-TTS distinctively outperforms our diffusion counterpart Grad-TTS.

\subsection{Preference Test}
Apart from using the MOS and CMOS tests to evaluate sample quality, we conducted a blind preference test when NFE=1000 and NFE=2, in order to demonstrate our superior generation quality and efficient sampling process, respectively. In each test, we generated 100 identical samples with two different models from the test set LJ001 and LJ002, and invited 11 judges to compare their overall subjective quality. The judge gives a preference when he thinks a model is better than the other, and an identical result when he thinks it is hard to tell the difference or the models have similar overall quality. In both preference tests, the settings of noise schedule, model parameterization and sampling process in Bridge-TTS are Bridge-gmax schedule with $\beta_0=0.01, \beta_1=50$, $\x_0$ predictor, and first-order SDE sampler with $\tau_b=2$, respectively.

In the case of NFE=1000, as shown in Figure \ref{fig:preferencetest} (a), when Bridge-TTS-1000 is compared with our diffusion counterpart Grad-TTS-1000 \citep{Grad-TTS} (temperature $\tau_d=1.5$), 8 of the 11 invited judges vote for Bridge-TTS-1000, and 3 of them think the overall quality is similar. In our blind test, none of the 11 judges preferred Grad-TTS-1000 to Bridge-TTS-1000. The comparison result is aligned with the MOS test shown in Table and CMOS test shown in Table \ref{tab:cmosqualitycomparison}.

In the case of NFE=2, as shown in Figure \ref{fig:preferencetest} (b), when Bridge-TTS-2 is compared with state-of-the-art fast sampling method in diffusion-based TTS systems, CoMoSpeech (1-step generation) \citep{CoMoSpeech}, 9 of the 11 invited judges vote for Bridge-TTS-2, and 2 of the judges vote for CoMoSpeech-1. Although Bridge-TTS employs 2 sampling steps while CoMoSpeech-1 only uses 1, the RTF of both methods have been very small (0.007 for CoMoSpeech-1 vs 0.009 for Bridge-TTS-2), and Bridge-TTS does not require any distillation process. According to our collected feedback, 9 judges think the overall quality (e.g., quality, naturalness, and accuracy) of Bridge-TTS is significantly better.

\section{Generated Samples}
\label{generatedsamples}

With the pre-trained HiFi-GAN \citep{HiFi-GAN} vocoder, we show the 80-band mel-spectrogram of several synthesized test samples of baseline models and our Bridge-TTS (Bridge-gmax schedule with $\beta_0=0.01, \beta_1=50$, $\x_0$ predictor, and first-order SDE sampler with $\tau_b=2$) below. The mel-spectrogram of ground-truth recording is shown for comparison. More generated speech samples can be visited on our website: \url{https://bridge-tts.github.io/}.
\begin{figure}[t!]
    \centering
    \includegraphics[width=0.95\textwidth]{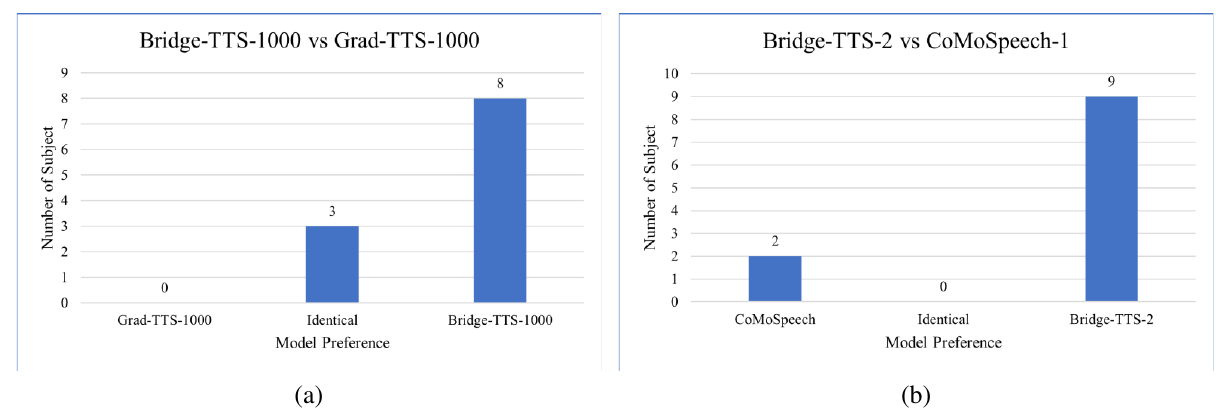}
    \caption{The preference test between Bridge-TTS and diffusion-based TTS systems.}
    \label{fig:preferencetest}
\end{figure}

\newpage
\paragraph{1000-step generation}

As exhibited in Figure \ref{fig:t1000sample0006} and Figure \ref{fig:t1000sample0029}, when NFE=1000, our method generates higher-quality speech than Grad-TTS (temperature $\tau_d=1.5$) built on data-to-noise process, demonstrating the advantage of our proposed data-to-data process over data-to-noise process in TTS.   

\begin{figure}[b!]
    \centering
    \subfloat[\label{fig:t1000sample0006}]{
    \includegraphics[width=0.75\textwidth]{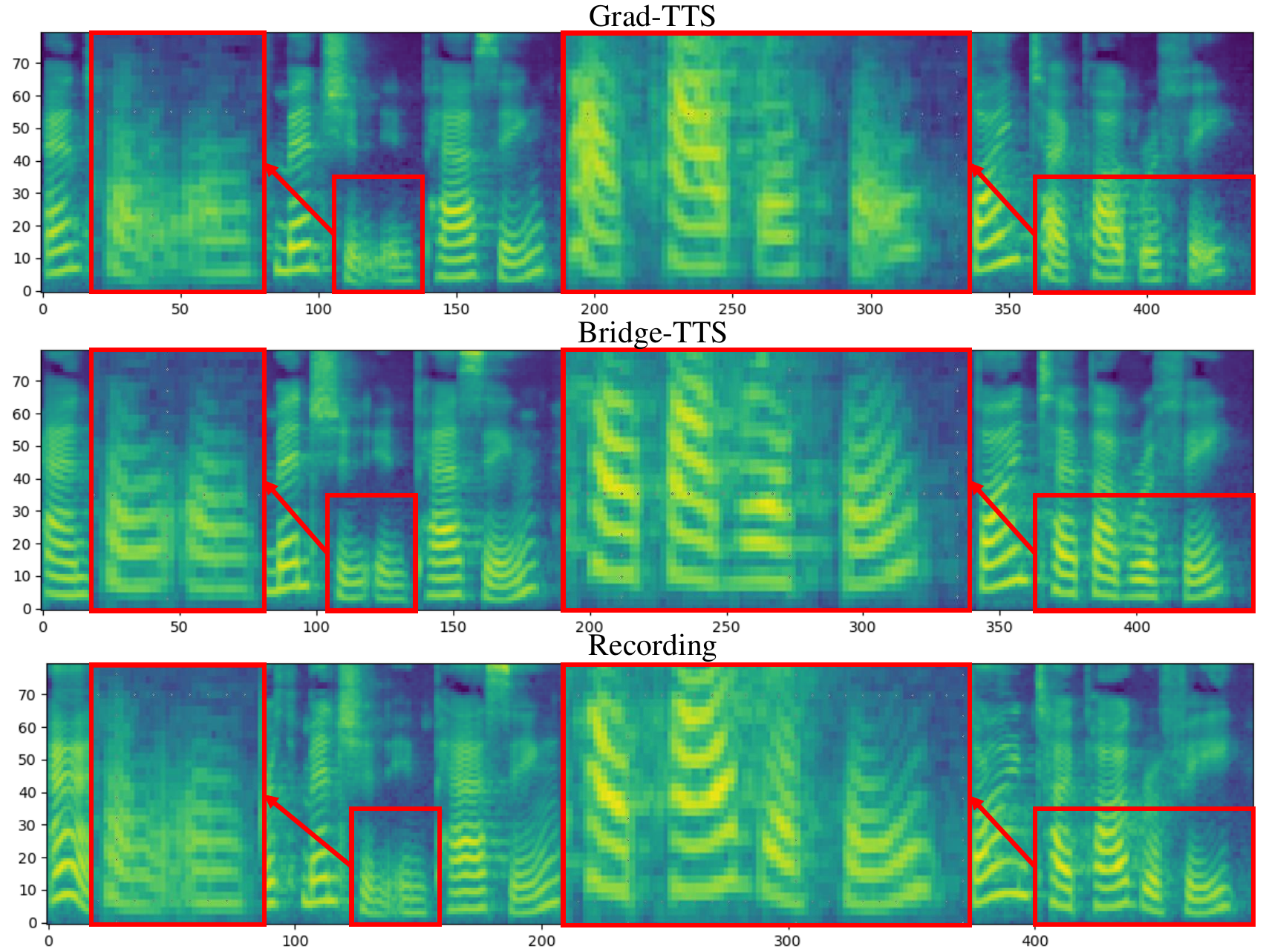}}
	\\
    \subfloat[\label{fig:t1000sample0029}]{
    \includegraphics[width=0.75\textwidth]{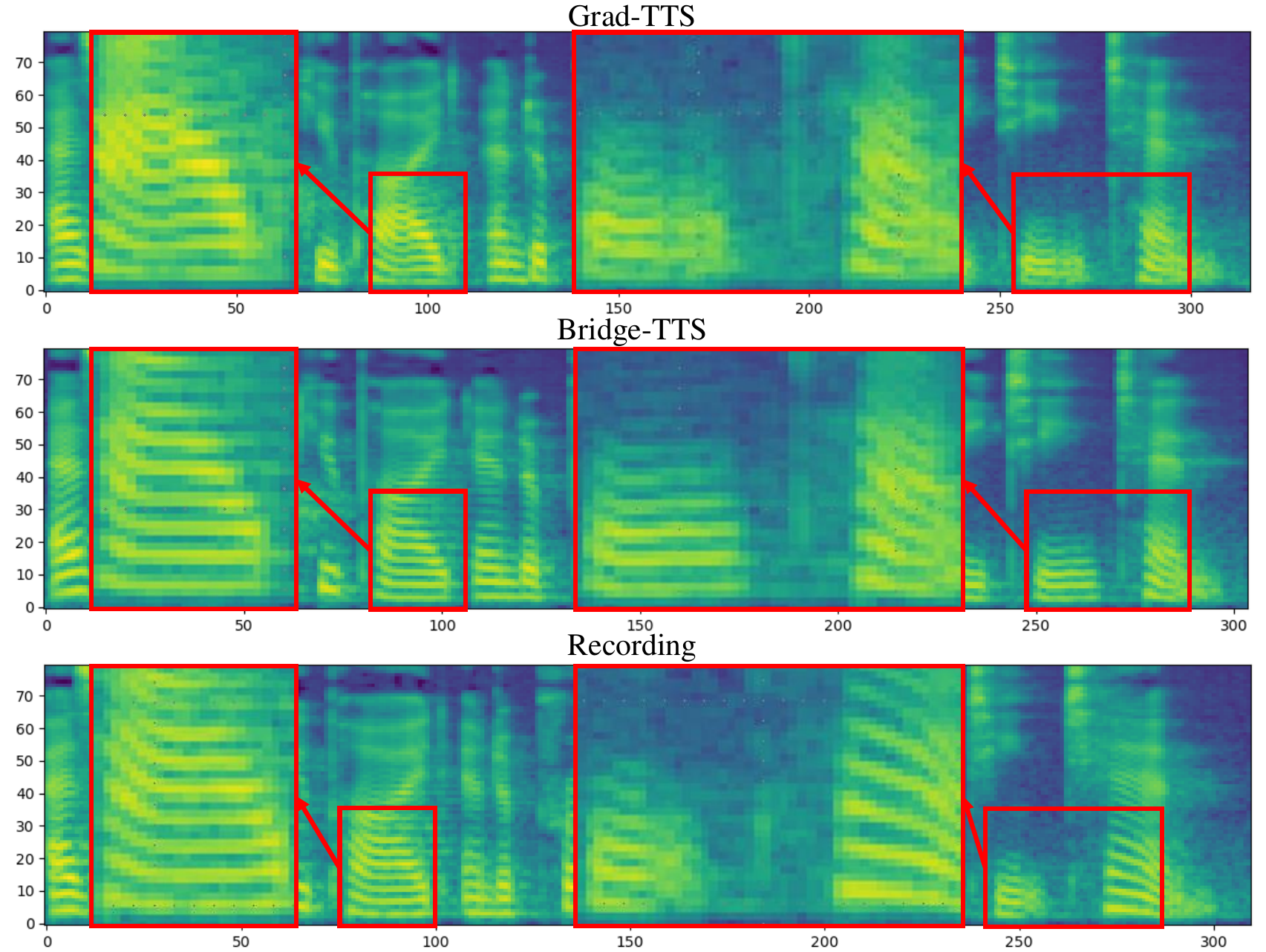} }
    \caption{The mel-spectrogram of synthesized (NFE=1000) and ground-truth LJ001-0006 and LJ002-0029.}
    \label{fig3} 
\end{figure}

\newpage
\paragraph{50-Step Generation}
In Figure \ref{fig:t50sample0035}, our method shows higher generation quality than Grad-TTS \citep{Grad-TTS}. In Figure \ref{fig:t50sample0029}, we continue to use the test sample LJ002-0029 to demonstrate our performance. In comparison with NFE=1000 shown in Figure \ref{fig:t1000sample0029}, when reducing NFE from 1000 to 50, Grad-TTS generates fewer details and sacrifices the sample quality, while our method still generates high-quality samples.

\begin{figure}[b!]
    \centering
    \subfloat[\label{fig:t50sample0035}]{
    \includegraphics[width=0.73\textwidth]{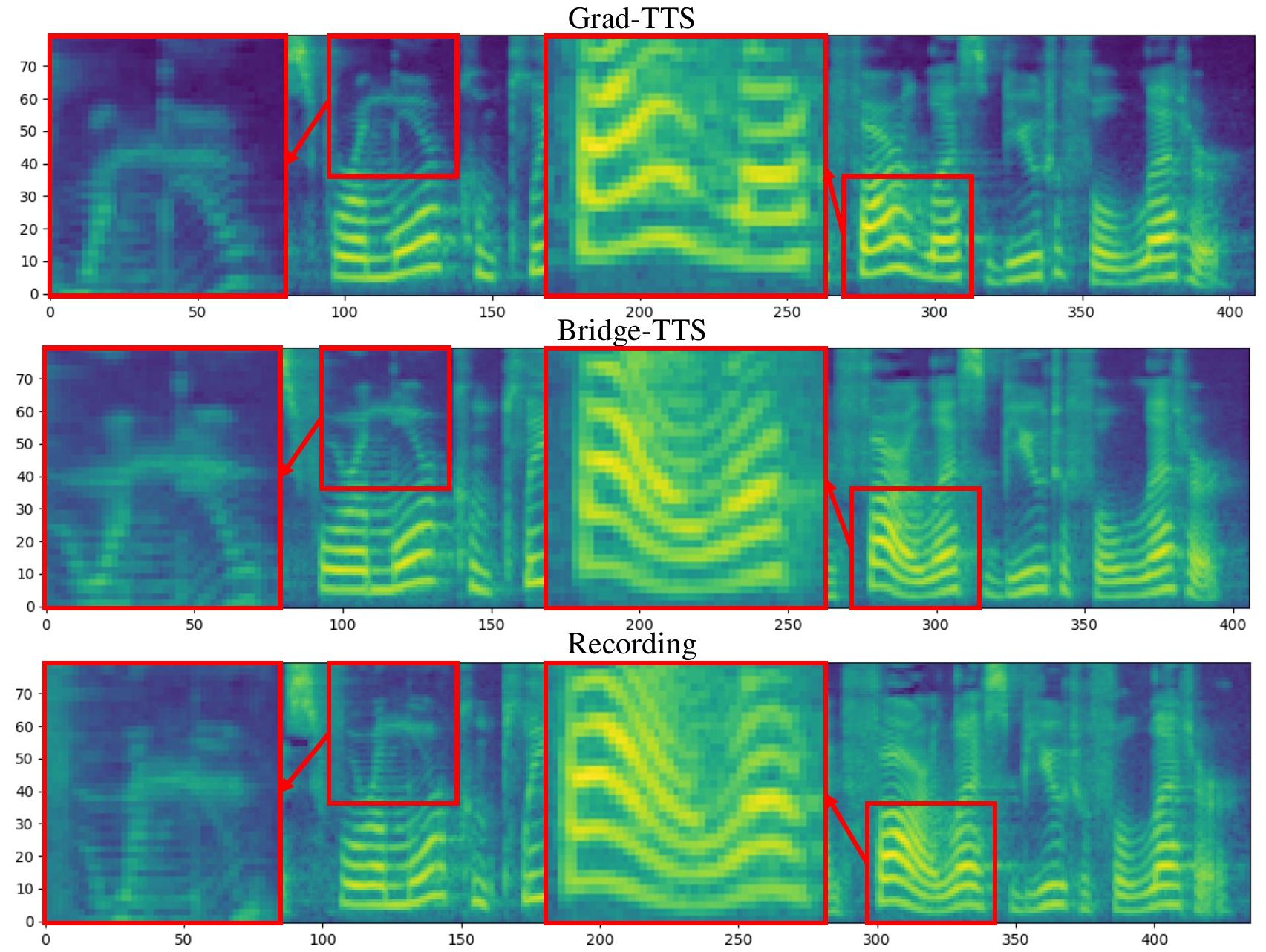}}
	\\
    \subfloat[\label{fig:t50sample0029}]{
    \includegraphics[width=0.73\textwidth]{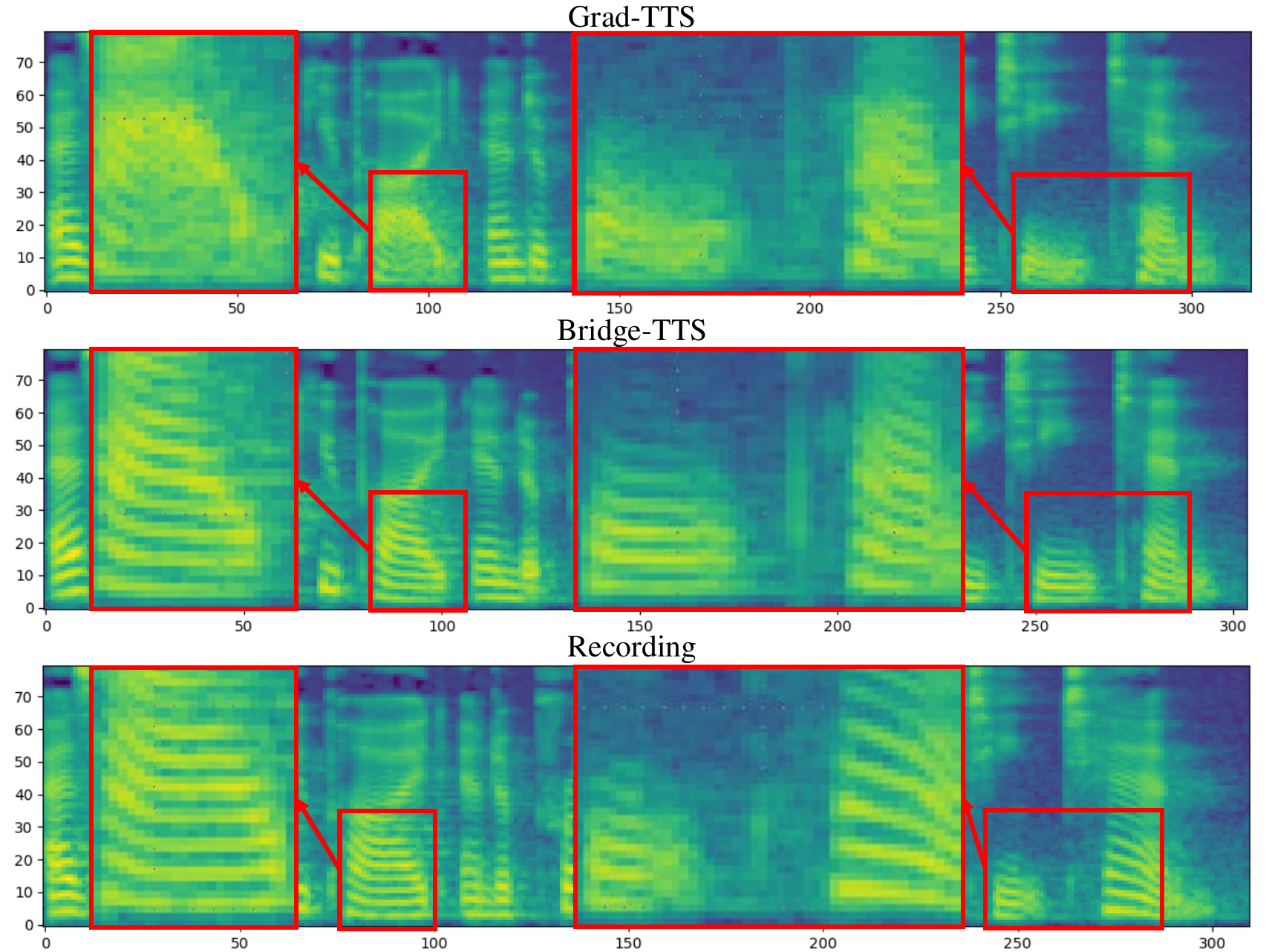} }
    \caption{The mel-spectrogram of synthesized (NFE=50) and ground-truth LJ001-0035 and LJ002-0029.}
    \label{fig3} 
\end{figure}

\newpage
\paragraph{4-Step Generation}
In Figure \ref{fig:t4sample0032}, we show our comparison with two baseline models, \textit{i.e.}, Grad-TTS \citep{Grad-TTS} and FastGrad-TTS \citep{FastGrad-TTS}. The latter one employs a first-order maximum-likelihood solver \citep{GradVC} for the pre-trained Grad-TTS, and reports stronger quality than Grad-TTS in 4-step synthesis. In our observation, when NFE=4, FastGrad-TTS achieves higher quality than Grad-TTS, while our method Bridge-TTS achieves higher generation quality than both of them, demonstrating the advantage of our proposed data-to-data process on sampling efficiency in TTS synthesis.
\begin{figure}[hb]
    \centering
    \includegraphics[width=0.75\textwidth]{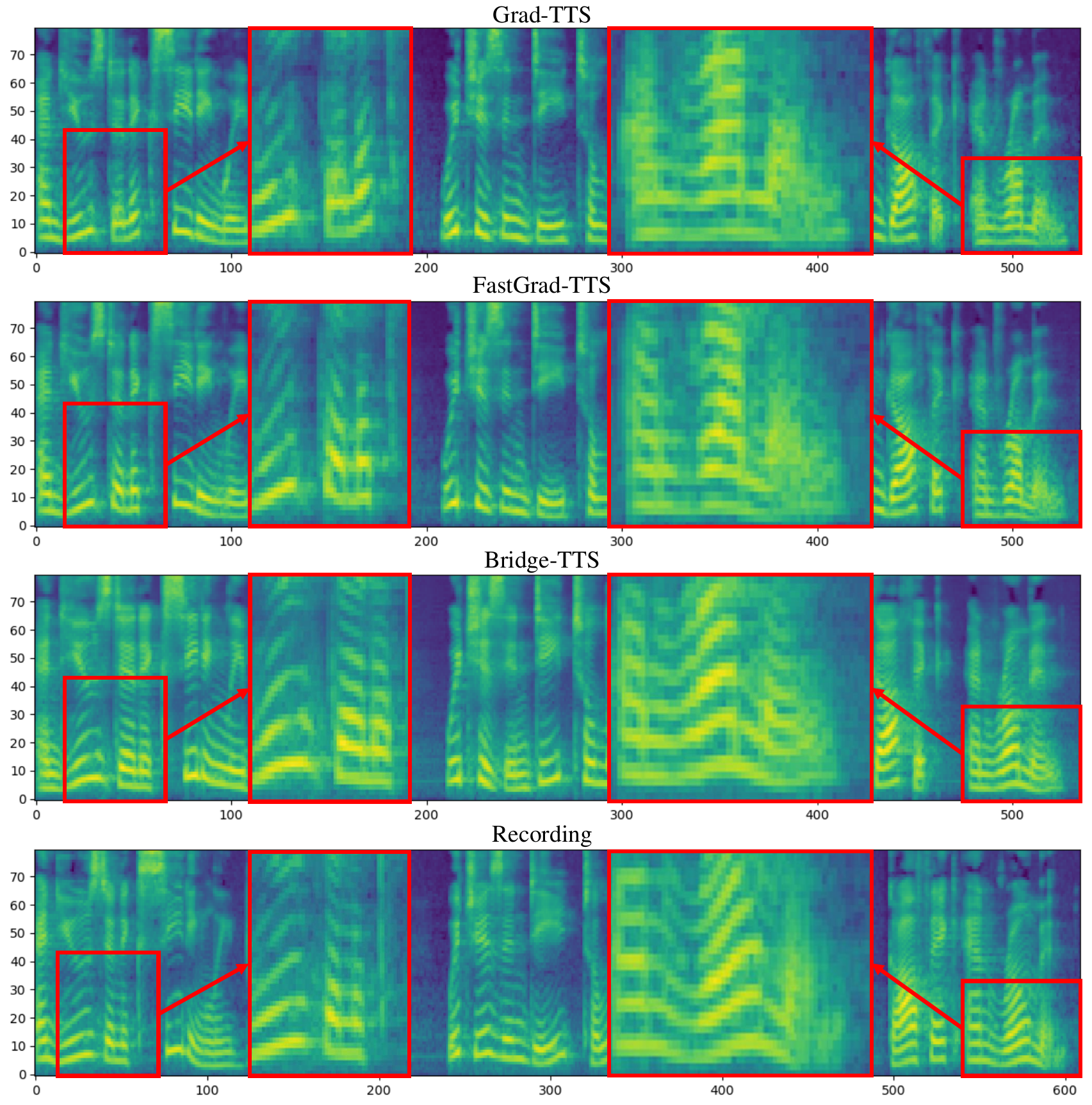}
    \caption{The mel-spectrogram of synthesized (NFE=4) and ground-truth sample LJ001-0032.}
    \label{fig:t4sample0032}
\end{figure}

\newpage
\paragraph{2-Step Generation}
When NFE is reduced to 2, we compare our method Bridge-TTS with the transformer-based model FastSpeech 2 \citep{Fastspeech2} and two diffusion-based TTS systems using distillation techniques. ProDiff \citep{ProDiff} employs progressive distillation achieving 2-step generation. CoMoSpeech \citep{CoMoSpeech} employs consistency distillation achieving 1-step generation. In our observation, in this case, the RTF of each model has been very small, and the overall generation quality is reduced. In the subjective test, our Bridge-TTS outperforms the other three methods. We show a short test sample, LJ001-0002, in Figure \ref{fig:t2sample0002}.

\begin{figure}[htbp]
    \centering
    \includegraphics[width=0.75\textwidth]{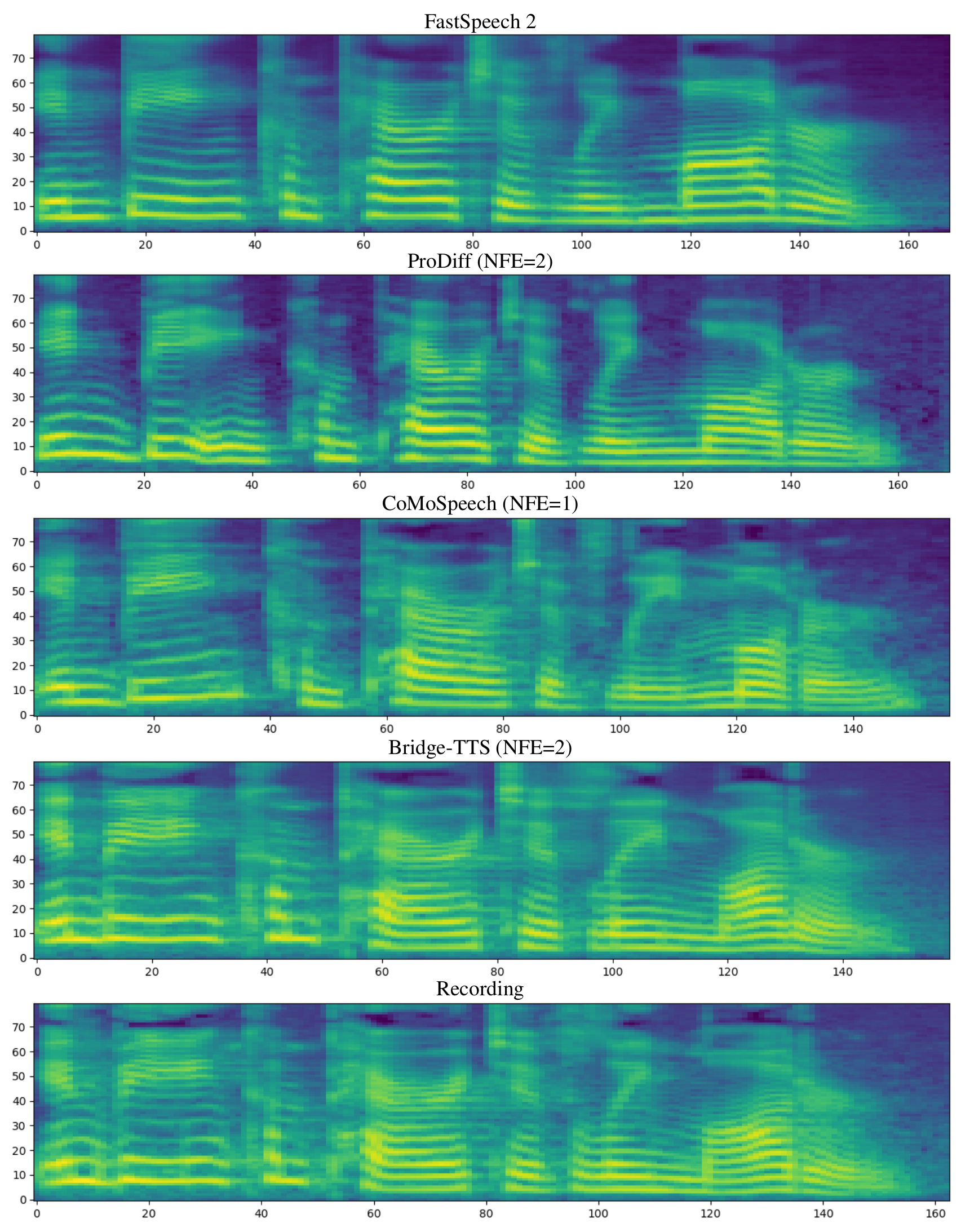}
    \caption{The mel-spectrogram of synthesized (NFE$\leq$2) and ground-truth sample LJ001-0002.}
    \label{fig:t2sample0002}
\end{figure}

\end{document}